\documentclass{article}

\usepackage[final]{nips_2018}

\usepackage[utf8]{inputenc} 
\usepackage[T1]{fontenc}    
\usepackage{hyperref}       
\usepackage{url}            
\usepackage{booktabs}       
\usepackage{amsfonts}       
\usepackage{nicefrac}       
\usepackage{microtype}      

\usepackage{amssymb,amsthm,amsmath}
\usepackage{array,graphicx}
\usepackage{comment}
\usepackage{wrapfig}

\usepackage{algorithm}
\usepackage[noend]{algpseudocode}
\usepackage{amsmath} 
\usepackage{amssymb} 
\usepackage{mathtools} 
\usepackage{stmaryrd}
\usepackage{bm}
\usepackage{tikz}
\usetikzlibrary{calc}
\usetikzlibrary{decorations}
\usepackage{multirow}

\usepackage{thmtools, thm-restate}

\usepackage{tikz}
\usepackage{pgfplots}
\usetikzlibrary{arrows}
\usetikzlibrary{positioning}
\usetikzlibrary{decorations.text}
\usetikzlibrary{decorations.markings}
\usetikzlibrary{decorations.shapes}
\usetikzlibrary{shapes,snakes}

\newtheorem{theorem}{Theorem}
\newtheorem{proposition}{Proposition}

\newcommand{\Algo}[1]{\textsc{#1}}

\renewcommand{\vec}[1]{\boldsymbol{#1}}

\newcommand{\bx}{\vec{x}}
\newcommand{\by}{\vec{y}}
\newcommand{\bz}{\vec{z}}
\newcommand{\bw}{\vec{w}}
\newcommand{\bh}{\vec{h}}
\newcommand{\bX}{\vec{X}}
\newcommand{\bY}{\vec{Y}}

\newcommand{\calX}{\mathcal{X}}
\newcommand{\calY}{\mathcal{Y}}
\newcommand{\calC}{\mathcal{C}}

\newcommand{\calL}{\mathcal{L}}
\newcommand{\calR}{\mathcal{R}}

\newcommand{\heta}{\hat{\eta}}
\newcommand{\hy}{\hat{y}}

\newcommand{\prob}{\mathbf{P}}
\newcommand{\hprob}{\widehat\prob}
\newcommand{\reg}{\mathrm{reg}}
\newcommand{\loss}{L}

\newcommand{\mM}{\boldsymbol{M}}
\newcommand{\mW}{\boldsymbol{W}}

\newcommand{\datasix}[6]{
    \begin{tabular}{ll}
    \rule{0pt}{0pt} #1 \\
    \rule{0pt}{0pt} #2 \\
    \rule{0pt}{0pt} #3 \\
    \rule{0pt}{0pt} #4 \\
    \rule{0pt}{0pt} #5 \\
    \rule{0pt}{0pt} #6 \\
    \end{tabular}
}

\newcommand{\data}[3]{
    \begin{tabular}{ll}
    \rule{0pt}{0pt} #1 \\
    \rule{0pt}{0pt} #2 \\
    \rule{0pt}{0pt} #3 \\
    \end{tabular}
}

\newcommand{\datatwo}[3]{
    \begin{tabular}{cc}
    \rule{0pt}{0pt} #1 \\
    \rule{0pt}{0pt} #2 \\
    \end{tabular}
}

\newcommand{\databf}[3]{
    \textbf{
	\begin{tabular}{ll}
    \rule{0pt}{0pt} #1 \\
    \rule{0pt}{0pt} #2 \\
    \rule{0pt}{0pt} #3 \\
    \end{tabular}
	}
}

\newcommand{\assert}[1]{\llbracket #1 \rrbracket}

\newcommand{\given}{\, | \,}

\DeclareMathOperator*{\argmin}{\arg \min}

\newcommand{\sectionBefore}{-0pt}
\newcommand{\sectionAfter}{-0pt}

\newcommand{\tableBefore}{-0pt}
\newcommand{\tableAfter}{-0pt}

\title{A no-regret generalization of hierarchical softmax to extreme multi-label classification}

\author{
   Marek Wydmuch\\
   Institute of Computing Science\\
   Poznan University of Technology, Poland\\
   \texttt{mwydmuch@cs.put.poznan.pl}\\
   \And
   Kalina Jasinska\\
   Institute of Computing Science\\
   Poznan University of Technology, Poland\\
   \texttt{kjasinska@cs.put.poznan.pl}\\
   \And
   Mikhail Kuznetsov\\
   Yahoo! Research\\
   New York, USA\\
   \texttt{kuznetsov@oath.com}\\
   \And
   R\'{o}bert Busa-Fekete \\
   Yahoo! Research\\
   New York, USA\\
   \texttt{busafekete@oath.com}\\
   \And
   Krzysztof Dembczy\'nski\\
   Institute of Computing Science\\
   Poznan University of Technology, Poland\\
   \texttt{kdembczynski@cs.put.poznan.pl}\\
} 

\begin{document}
\maketitle


\begin{abstract}
Extreme multi-label classification (XMLC) is a problem of tagging an instance with a small subset of relevant labels chosen from an extremely large pool of possible labels. Large label spaces can be efficiently handled by organizing labels as a tree, like in the hierarchical softmax (\Algo{HSM}) approach commonly used for multi-class problems. In this paper, we investigate probabilistic label trees (\Algo{PLT}s) that have been recently devised for tackling XMLC problems. 
We show that \Algo{PLT}s are a \emph{no-regret} multi-label generalization of \Algo{HSM} when precision@$k$ is used as a model evaluation metric. 
Critically, we prove that \emph{pick-one-label} heuristic---a reduction technique from multi-label to multi-class that is routinely used along with HSM---is not consistent in general. 
We also show that our implementation of \Algo{PLT}s, referred to as \Algo{extremeText} (\Algo{XT}), obtains significantly better results than \Algo{HSM} with the pick-one-label heuristic and \Algo{XML-CNN}, a deep network specifically designed for XMLC problems. Moreover, \Algo{XT} is competitive to many state-of-the-art approaches in terms of statistical performance, model size and prediction time which makes it amenable to deploy in an online system. 
%
\end{abstract}

\section{Introduction}
\label{sec:introduction}

In several machine learning applications, the label space can be enormous, containing even millions of different classes. Learning problems of this scale are often referred to as \emph{extreme classification}.  
To name a few examples of such problems, consider image and video annotation for multimedia search~\citep{Deng_et_al_2011}, tagging of text documents for categorization of Wikipedia articles~\citep{Dekel_Shamir_2010}, recommendation of bid words for online ads~\citep{Prabhu_Varma_2014}, or prediction of the next word in a sentence~\citep{Mikolov_et_al_2013}. 

To tackle extreme classification 
problems in an efficient way, one can organize the labels into a tree. A prominent example of such label tree model is \emph{hierarchical softmax} (\Algo{HSM})~\citep{Morin_Bengio_2005}, often used with neural networks to speed up computations in multi-class classification with
large output spaces. For example, it is commonly applied in natural language processing problems such as language modeling~\citep{Mikolov_et_al_2013}. To adapt \Algo{HSM} to \emph{extreme multi-label classification} (XMLC), several very popular tools, such as \Algo{fastText}~\citep{Joulin_et_al_2016} and \Algo{Learned Tree}~\citep{Jernite_et_al_2017}, apply the \emph{pick-one-label} heuristic. As the name suggests, this heuristic randomly picks one of the labels from a multi-label training example and treats the example as a multi-class one.

%

In this work, we exhaustively investigate the multi-label extensions of \Algo{HSM}. First, we show that the pick-one-label strategy does not lead to a proper generalization of HSM for multi-label setting.  
More precisely, we prove that using the pick-one-label reduction one cannot expect any multi-class learner to achieve zero regret in terms of marginal probability estimation and maximization of precision@$k$. 
As a remedy to this issue, we are going to revisit \emph{probabilistic label trees} (\Algo{PLT}s)~\citep{Jasinska_et_al_2016} that have been recently introduced for solving XMLC problems. 
We show that \Algo{PLT}s are a theoretically motivated generalization of \Algo{HSM} to multi-label classification, that is, 1) \Algo{PLT}s and \Algo{HSM} are identical in multi-class case, and 2) a \Algo{PLT} model can get \emph{zero regret} (i.e., it is \emph{consistent}) in terms of marginal probability estimation and precision@$k$ in the multi-label setting. 
%


Beside our theoretical findings, we provide an efficient implementation of \Algo{PLT}s, referred to as \Algo{XT}, that we build upon \Algo{fastText}. The comprehensive empirical evaluation shows that it gets significantly better results than the original \Algo{fastText}, \Algo{Learned Tree}, and \Algo{XML-CNN}, a specifically designed deep network for XMLC problems.
\Algo{XT} also achieves competitive results to other state-of-the-art approaches, being very efficient in model size and prediction time, particularly in the online setting.

This paper is organized as follows. First we discuss the related work and situate our approach in the context. In Section~\ref{sec:problem_statement} we formally state the XMLC problem and present some useful theoretical insights. Next, we briefly introduce the \Algo{HSM} approach, and in Section \ref{sec:pick-one-label} we show theoretical results concerning the pick-one-label heuristic. Section~\ref{sec:plt} formally introduces \Algo{PLT}s and presents the main theoretical results concerning them and their
relation to \Algo{HSM}. Section~\ref{sec:plt-tagging} provides implementation details of \Algo{PLT}s. 
The experimental results are presented in
Section~\ref{sec:empirical_results}. 
Finally we make concluding remarks. 

\vspace{\sectionBefore}
\section{Related work}
\label{sec:related_work}
\vspace{\sectionAfter}


Historically, problems with a large number of labels were usually
solved by nearest neighbor or decision tree methods. Some of today's
algorithms are still based on these classical approaches,
significantly extending them by a number of new tricks. If the label
space is of moderate size (like a few thousands of labels) then an
independent model can be trained for each class. This is the so-called
\Algo{1-vs-All} approach. Unfortunately, it scales linearly with the
number of labels, which is too costly for many applications. The
extreme classification algorithms try to improve over this approach
by following different paradigms such as
sparsity of labels~\citep{Yen_et_al_2017,Babbar_Scholkopf_2017}, low-rank
approximation~\citep{Mineiro_Karampatziakis_2015,Yu_et_al_2014,Bhatia_et_al_2015},
tree-based
search~\citep{Prabhu_Varma_2014,Choromanska_Langford_2015}, or label
filtering~\citep{Vijayanarasimhan_et_al_2014,Shrivastava_Li_2015,Niculescu-Mizil_Abbasnejad_2017}.

In this paper we focus on tree-based algorithms, therefore we discuss them here in more detail. 
There are two distinct types of these algorithms: decision trees and label trees. 
The former type follows the idea of classical decision trees. 
However, the direct use of the classic algorithms can be very costly~\citep{Agrawal_et_al_2013}. 
Therefore, the \Algo{FastXML} algorithm~\citep{Prabhu_Varma_2014} tackles the problem in a slightly different way. 
It uses sparse linear classifiers in internal tree nodes to split the feature space. Each linear classifier is trained on two classes that are formed in a random way first and then reshaped by optimizing the normalized discounted cumulative gain. 
To improve the overall accuracy \Algo{FastXML} uses an ensemble of trees. 
This algorithm, like many other decision tree methods, works in a batch mode. 
\citet{Choromanska_Langford_2015} have succeeded to introduce a fully online decision tree algorithm 
that also uses linear classifiers in internal nodes of the tree.

In label trees each label corresponds to one and only one path from the root to a leaf. 
Besides \Algo{PLT}s and \Algo{HSM}, there exist several other instances of this approach, 
for example, filter trees~\citep{Beygelzimer_et_al_2009a,Li_Lin_2014} 
or label embedding trees~\citep{Bengio_et_al_2010}. 
It is also worth to underline that algorithms similar to \Algo{HSM} 
have been introduced independently in many different research fields, 
such as nested dichotomies~\citep{Fox_1997} in statistics, 
conditional probability estimation trees~\citep{Beygelzimer_et_al_2009b} in multi-class classification, 
multi-stage classifiers~\citep{Kurzynski_1988} in pattern recognition, 
and probabilistic classifier chains~\citep{Dembczynski_et_al_2010c} in multi-label classification under the subset 0/1 loss. 
All these methods have been jointly analyzed in~\citep{Dembczynski_et_al_2016}.

A still open problem in label tree approaches is the tree structure learning. 
\Algo{fastText}~\citep{Joulin_et_al_2016} uses \Algo{HSM} with a Huffman tree built on the label frequencies. 
\citet{Jernite_et_al_2017} have introduced a new algorithm, called \Algo{Learned Tree}, which combines \Algo{HSM} 
with a specific hierarchical clustering that reassigns labels to paths in the tree in a semi-online manner. \citet{Prabhu_et_al_2018} follows another approach in which a model similar to \Algo{PLT}s is trained in a batch mode and a tree is built by using recursively balanced k-means over the label profiles. In Section~\ref{sec:plt-tagging} we discuss this approach in more detail.

The \Algo{HSM} model is often used as an output layer in neural networks. The \Algo{fastText} implementation can also be viewed as a shallow architecture with one hidden layer that represents instances as averaged feature (i.e., word) vectors. 
Another neural network-based model designed for XMLC has been introduced in~\citep{Liu_et_al_2017}. This model, referred to as \Algo{XML-CNN}, uses a complex convolutional deep network with a narrow last layer to make it work with large output spaces.  As we show in the experimental part, this quite expensive architecture gets inferior results in comparison to our \Algo{PLT}s built upon \Algo{fastText}.  

%

\vspace{\sectionBefore}
\section{Problem statement}
\label{sec:problem_statement}
\vspace{\sectionAfter}

Let $\calX$ denote an instance space, and let $\calL = \{1, \ldots, m\}$ be a finite set of $m$ class labels. 
We assume that an instance $\bx \in \calX$ is associated with a subset of
labels $\calL_{\bx} \in 2^\calL$ (the subset can be empty); this subset is often called a set of relevant labels, while the complement
$\calL \backslash \calL_{\bx}$ is considered as irrelevant for $\bx$. We assume $m$ to be a large number (e.g., $\ge 10^5$), but the size of the set of relevant labels $\calL_{\bx}$ is much smaller than $m$, i.e., $|\calL_{\bx}| \ll m$. We identify a set $\calL_{\bx}$ of relevant labels with a binary (sparse)
vector $\by = (y_1,y_2, \ldots, y_m)$, in which $y_j = 1 \Leftrightarrow j \in \calL_{\bx}$. By $\calY = \{0, 1\}^m$ we denote a set of all possible label vectors.
We assume that observations $(\bx, \by)$ are generated independently and identically according to the
probability distribution $\prob(\bX = \bx,\bY = \by)$ (denoted later by $\prob(\bx, \by)$) defined on $\calX \times \calY$. 

The problem of XMLC can be defined as finding a \emph{classifier} $\bh(\bx) = (h_1(\bx), h_2(\bx),\ldots, h_m(\bx))$, 
which in general can be defined as a mapping $\calX \rightarrow \calR^m$, that minimizes the \emph{expected loss} (or \emph{risk}):  
$$
\loss_\ell(\bh) = \mathbb{E}_{(\bx,\by) \sim \prob(\bx,\by)} (\ell(\by, \bh(\bx))\,,
$$
where $\ell(\by, \hat{\by})$ is the  (\emph{task}) \emph{loss}.
The optimal classifier,  the so-called \emph{Bayes classifier},  for a given loss function $\ell$ is:
$$
\bh^*_\ell = \argmin_{\bh}  \loss_\ell(\bh) \,.
$$
The \emph{regret} of a classifier $\bh$ with respect to $\ell$ is defined as:
 $$
\reg_\ell(\bh) = \loss_\ell(\bh) - \loss_\ell(\bh_{\ell}^*) = \loss_\ell(\bh) - \loss_\ell^* \,.
$$
The regret quantifies the suboptimality of $\bh$ compared to the optimal classifier $\bh^*$. The goal could be then defined as finding $\bh$ with a small regret, ideally equal to zero.

In the following, we aim at estimating the marginal probabilities $\eta_j(\bx) = \prob(y_j = 1 \given \bx)$. 
%
As we will show below, marginal probabilities are a key element to optimally solve extreme classification for many performance measures, like Hamming loss, macro-F measure, and precision@$k$. 
To obtain the marginal probability estimates one can use the label-wise log loss as a surrogate:
$$
\ell_{\log}(\by, \bh(\bx))  = \sum_{j=1}^m \ell_{\log}(y_j, h_j(\bx)) = \sum_{j=1}^m  \left ( y_j \log(h_j(\bx)) + (1-y_j) \log(1-h_j(\bx)) \right) \,.
$$
Then the expected label-wise log loss for a single $\bx$ (i.e., the so-called \emph{conditional risk}) is:
$$
\mathbb{E}_{\by} \ell_{\log}(\by, \bh(\bx)) =  \sum_{j=1}^m \mathbb{E}_{\by}{\ell_{\log}(y_j, h_j(\bx))} = \sum_{j=1}^m \loss_{\log}(h_j(\bx) \given \bx)\,. 
$$
Therefore, it is easy to see that the pointwise optimal prediction for the $j$-th label is given by:
$$
h_j^*(\bx)  = \argmin_h \loss_{\log}(h_j(\bx)\given \bx) = \eta_j(\bx) \,.
$$

As shown in~\citep{Dembczynski_et_al_2010c}, the Hamming loss is minimized by 
$h_j^*(\bx) = \assert{\eta_j(\bx) > 0.5} \,.$
For the macro F-measure it suffices in turn to find an optimal threshold on marginal probabilities for each label separately as proven in~\citep{Ye_et_al_2012,Narasimhan_et_al_2014,Jasinska_et_al_2016,Dembczynski_et_al_2017}. In the following, we will show a similar result for precision@$k$ which has become a standard measure in extreme classification (although it is also often criticized, as it favors the most frequent labels). 
%

Precision@$k$ can be formally defined as:
\begin{equation}
\mathrm{precision}@k(\by, \bx, \bh) = \frac{1}{k} \sum_{j \in \hat \calY_k} \assert{y_j = 1},
\label{eqn:precision-at-k}
\end{equation}
where $\hat \calY_k$ is a set of $k$ labels predicted by $\bh$ for $\bx$.
To be consistent with the former discussion, let us define a loss function for precision@$k$ as $\ell_{p@k} = 1 - \mathrm{precision}@k$. The conditional risk is then:\footnote{The derivation is given in Appendix~\ref{app:prec@k}.}
$$
\loss_{p@k}(\bh \given \bx) = \mathbb{E}_{\by} \ell_{p@k}(\by,\bx, \bh) = 1 - \frac{1}{k} \sum_{j \in \hat \calY_k} \eta_j(\bx) \,.
$$
The above result shows that the optimal strategy for precision@$k$ is to predict $k$ labels
with the highest marginal probabilities $\eta_j(\bx)$.
As the main theoretical result given in this paper is a regret bound for precision@$k$, let us define here the conditional regret for this metric:
$$
\reg_{p@k} (\bh \given \bx) = \frac{1}{k}\sum_{i \in \calY_k} \eta_i(\bx) - \frac{1}{k}\sum_{j \in \hat \calY_k} \eta_j(\bx)\,,
$$
where $\calY_k$ is a set containing the top $k$ labels with respect to the true marginal probabilities. 

From the above results, we see that estimation of marginal probabilities is crucial for XMLC problems. To obtain these probabilities we can use the vanilla \Algo{1-vs-All} approach trained with the label-wise log loss. Unfortunately, \Algo{1-vs-All} is too costly in the extreme setting. In the following sections, we discuss an alternative approach based on the label trees that estimates the marginal probabilities with the competitive accuracy, but in a much more efficient way.

\vspace{\sectionBefore}
\section{Hierarchical softmax approaches}
\label{sec:hsm}
\vspace{\sectionAfter}

Hierarchical softmax (\Algo{HSM}) is designed for multi-class classification. Using our notation, for multi-class problems we have $\sum_{i=1}^m y_i = 1$, i.e., there is one and only one label assigned to an instance $(\bx, \by)$. The marginal probabilities $\eta_j(\bx)$ in this case sum up to 1. 


The \Algo{HSM} classifier $\bh(\bx)$ takes a form of a label tree. 
We encode all labels from $\calL$ using a prefix code. Any such code can be given in a form of a tree in which a path from the root to a leaf node corresponds to a code word. Under the coding, each label $y_j=1$ is uniquely represented by a code word $\bz = (z_1, \ldots, z_l) \in \calC$, where $l$ is the length of the code word and $\calC$ is a set of all code words. 
For $z_i \in \{0,1\}$, the code and the label tree are binary. In general, the code alphabet can contain more than two symbols. Furthermore, $z_i$s can take values from different sets of symbols depending on the previous values in the code word. In other words, the code can result with nodes of a different arity even in the same tree, like in~\citep{Grave_et_al_2017} and \citep{Prabhu_et_al_2018}. We will briefly discuss different tree structures in Section~\ref{sec:plt-tagging}. 

A tree node can be uniquely identified by the partial code word $\bz^i = (z_1, \ldots, z_i)$. We denote the root node by $\bz^0$, which is an empty vector (without any elements). The probability of a given label is determined by a sequence of decisions made by node classifiers that predict subsequent values of the code word. 
By using the chain rule of probability, we obtain:
$$
\eta_j(\bx) = \prob(y_j = 1 \given \bx) = \prob(\bz \given \bx) = \prod_{i=1}^l \prob(z_i \given \bz^{i-1}, \bx) \,.
$$

By using logistic loss and a linear model $f_{\bz^i}(\bx)$ in each node $\bz^i$ for estimating $\prob(z_i \given \bz^{i-1}, \bx)$, we obtain the popular formulation of \Algo{HSM}. 
Let us notice that since we deal here with a multi-class distribution, we have that:
\begin{equation}
\sum_{c} \prob(z_i = c \given \bz^{i-1}, \bx) = 1 \,.    
\label{eqn:normalization-hsm}
\end{equation}
Because of this normalization, we can assume that a multi-class (or binary in the case of binary trees) classifier is situated in all internal nodes and there are no classifiers in the leaves of the tree. Alternatively, we can assume that each node, except the root, is associated with a binary classifier that estimates $\prob(z_i = c \given \bz^{i-1}, \bx)$, but then the additional normalization (\ref{eqn:normalization-hsm}) has to be performed. This alternative formulation is important for the multi-label extension of HSM  discussed in Section~\ref{sec:plt}.
In either way, learning of the node classifiers can be performed simultaneously as independent tasks. 

Note that estimate $\heta_j(\bx)$ of the probability of label $j$ can be easily obtained by traversing the tree along the path indicated by the code of the label. Unfortunately, the task of predicting top $k$ labels is more involved as it requires searching over the tree. Popular solutions are beam search~\citep{Kumar_et_al_2013,Prabhu_et_al_2018}, uniform-cost search~\citep{Joulin_et_al_2016}, and its approximate variant~\citep{Dembczynski_et_al_2012c,Dembczynski_et_al_2016}.

\vspace{\sectionBefore}
\section{Suboptimality of HSM for multi-label classification}
\label{sec:pick-one-label}
\vspace{\sectionAfter}

To deal with multi-label problems, some popular tools, such as \Algo{fastText}~\citep{Joulin_et_al_2016} and its extension \Algo{Learned Tree}~\citep{Jernite_et_al_2017}, apply \Algo{HSM} with the pick-one-label heuristic which randomly picks one of the positive labels from a given training instance. The resulting instance is then treated as a multi-class instance. During prediction, the heuristic returns a multi-class distribution and the $k$ most probable labels. 
We show below that this specific reduction of the multi-label problem to multi-class classification is not consistent in general.

Since the probability of picking a label $j$ from $\by$ is equal to $y_j/\sum_{j'=1}^m y_{j'}$, the pick-one-label heuristic maps the multi-label distribution to a multi-class distribution in the following way:
\begin{equation}
\eta_j'(\bx) = \prob'(y_j = 1 \given \bx) = \sum_{\by \in \calY}  \frac{y_j}{\sum_{j'=1}^m y_{j'}}\prob(\by \given \bx)
\label{eq:heuristic}
\end{equation}
It can be easily checked that the resulting $\eta_j'(\bx)$ form a multi-class distribution as the probabilities sum up to 1. 
It is obvious that that the heuristic changes the marginal probabilities of labels, unless the initial distribution is multi-class. Therefore this method cannot lead to consistent classifiers in terms of estimating $\eta_j(\bx)$. As we show below, it is also not consistent for precision@$k$ in general. 
\begin{proposition}
A classifier $\bh$ such that $h_j(\bx) = \eta_j'(\bx)$ for all $j \in \{1, \dots ,m\}$ has in general a non-zero regret in terms of precision@$k$.
\end{proposition}
\begin{proof}
We prove the proposition by giving a simple counterexample. Consider the following conditional distribution for some $\bx$: 
$$
\prob(\by = (1,0,0) \given \bx) = 0.1 \,,\quad \prob(\by = (1,1,0) \given \bx) = 0.5\,, \quad \prob(\by = (0,0,1) \given \bx) = 0.4\,.
$$
The optimal top 1 prediction for this example is obviously label $1$, since the marginal probabilities are $\eta_1(\bx) = 0.6, \eta_2(\bx) = 0.5,  \eta_3(\bx) = 0.4$. However, the pick-one-label heuristic will transform the original distribution to the following one: $\eta_1'(\bx) = 0.35, \eta_2'(\bx) = 0.25,  \eta_3'(\bx) = 0.4$. The predicted top label will be then label $3$, giving the regret of 0.2 for precision@$1$.  
\end{proof}
%

The proposition shows that the heuristic is in general inconsistent for precision@$k$.
Interestingly, the situation changes when the labels are conditionally independent, i.e., 
$
\prob(\by \given \bx) = \prod_{j=1}^m \prob(y_i \given \bx)\,. 
$
\begin{restatable}{proposition}{hsmindependent}
\label{prop:hsm-independent}
Given conditionally independent labels, a classifier $\bh$ such that $h_j(\bx) = \eta_j'(\bx)$ for all $j \in \{1, \dots ,m\}$ has zero regret in terms of the precision@$k$ loss.
\end{restatable}
\begin{proof}
We show here only a sketch of the proof. The full proof is given in Appendix~\ref{app:hsm-independent}. To prove the theorem, it is enough to show that in the case of conditionally independent labels the pick-one-label heuristic does not change the order of marginal probabilities.
Let $y_i$ and $y_j$ be so that $\prob(y_i = 1 \given \bx) \ge \prob(y_j = 1 \given \bx) $. Then in the summation over all $\by$s in (\ref{eq:heuristic}), we are interested in four different subsets of $\calY$,
$S_{i,j}^{u,w}  =  \{\by\in \calY: y_i = u \land y_j = w\}$, where $u,w \in \{0,1\}$.
Remark that during mapping none of $\by \in S^{0,0}_{i,j}$ plays any role, and for each $\by \in S^{1,1}_{i,j}$, the value of 
$
y_t/(\sum_{t'=1}^m y_{t'}) \times \prob(\by \given \bx) 
$, for $t \in \{i,j\}$, 
is the same for both $y_i$ and $y_j$. Now, let $\by' \in S^{1,0}_{i,j}$ and $\by'' \in S^{0,1}_{i,j}$ be the same on all elements except the $i$-th and the $j$-th one. Then, because of the label independence and the assumption that $\prob(y_i = 1 \given \bx) \ge \prob(y_j = 1 \given \bx) $, we have $\prob(\by' \given \bx) \ge \prob(\by'' \given \bx)$. Therefore, after mapping we obtain $\eta_i'(\bx) \ge \eta_j'(\bx)$. 
Thus, for independent labels, the pick-one-label heuristic is consistent for precision@$k$.
\end{proof}


%
%
%


\vspace{\sectionBefore}
\section{Probabilistic label trees}
\label{sec:plt}
\vspace{\sectionAfter}

The section above has revealed that \Algo{HSM} cannot be properly adapted to multi-label problems by the pick-one-label heuristic. There is, however, a different way to generalize \Algo{HSM} to obtain no-regret estimates of marginal probabilities $\eta_j(\bx)$. 
The probabilistic label trees (\Algo{PLT}s)~\citep{Jasinska_et_al_2016} can be derived in the following way. Let us encode $y_j = 1$ by a slightly extended code $\bz = (1, z_1, \ldots, z_l)$ in comparison to \Algo{HSM}. The new code gets 1 at the zero position what corresponds to a question whether there exists at least one label assigned to the example. As before, each node is uniquely identified by a partial code $\bz^i$ which says that there is at least one positive label in a subtree rooted in that node. 
It can be easily shown by the chain rule of probability that the marginal probabilities can be expressed in the following way:
\begin{equation}
\eta_j(\bx) = \prob(\bz \given \bx) = \prod_{i = 0}^l \prob(z_i \given \bz^{i-1}, \bx)\,.
\label{eqn:plt}
\end{equation}
The difference to \Algo{HSM} is the probability $\prob(z_0 = 1 \given \bx)$ in the chain and a different normalization, i.e.:
\begin{equation}
\sum_c \prob(z_i = c \given \bz^{i-1}, \bx) \ge 1 \,.
\label{eqn:normalization-plt}
\end{equation}
Only for $z_0$ we have $\prob(z_0 = 1 \given \bx) + \prob(z_0 = 0 \given \bx) = 1$. Because of (\ref{eqn:normalization-plt}), 
the binary models that estimate $\prob(z_i = c \given \bz^{i-1}, \bx)$ (against $\prob(z_i \not = c \given \bz^{i-1}, \bx)$) are situated in all nodes of the tree (i.e., also in the leaves). The models can be trained independently as before for \Algo{HSM}. Only during prediction, one can re-calibrate the estimates when (\ref{eqn:normalization-plt}) is not satisfied, for example, by normalizing them to sum up to 1. 
It can be easily noticed that for a multi-class distribution, the resulting model of \Algo{PLT}s boils down to \Algo{HSM}, since $\prob(z_0 = 1 \given \bx)$ is always equal $1$, and in addition, normalization (\ref{eqn:normalization-plt}) will take the form of (\ref{eqn:normalization-hsm}).
In Appendix~\ref{app:pseudocode} we additionally present the pseudocode of training and predicting with \Algo{PLT}s.

Next, we show that the \Algo{PLT} model obeys strong theoretical guarantees. Let us first revise the result from~\citep{Jasinska_et_al_2016} that relates the absolute difference between the true and the estimated marginal probability of label~$j$, $|\eta_j(\bx) - \heta_j(\bx)|$, to the surrogate loss $\ell$ used to train node classifiers $f_{\bz^i}$. It is assumed here that $\ell$ is a strongly proper composite loss (e.g, logistic, exponential, or squared loss) characterized by a constant $\lambda$ (e.g. $\lambda = 4$ for logistic loss).\footnote{For more detailed introduction to strongly proper composite losses, we refer the reader to~\citep{Agarwal_2014}.}
\begin{theorem}\label{thm:log}
For any distribution $\prob$ and internal node classifiers $f_{\bz^i}$, the following holds:
$$
\left | \eta_j(\bx) - \heta_j(\bx) \right |  \le  \sum_{i=0}^l \prob(\bz^{i-1} \given \bx) \sqrt{ \frac{2}{\lambda}} \sqrt{\reg_\ell(f_{\bz^i} \given \bz^{i-1}, \bx)} \,,
$$
where $\reg_\ell(f_{\bz^i} \given \bz^{i-1}, \bx)$ is a binary classification regret for a strongly proper composite loss $\ell$ and $\lambda$ is a constant specific for loss $\ell$.
\end{theorem}
Due to filtering of the distribution imposed by the \Algo{PLT}, the regret $\reg_\ell(f_{\bz^i} \given \bz^{i-1}, \bx)$ of a classifier $f_{\bz^i}$ exists only for $\bx$ such that $\prob(\bz^{i-1} \given \bx) > 0$, therefore we condition the regret not only on $\bx$, but also on $\bz^{i-1}$.
The above result shows that the absolute error of estimating the marginal probability of label $j$ can be upper bounded by the regret of the node classifiers on the corresponding path from the root to a leaf. 
The proof of Theorem \ref{thm:log} is given in Appendix~\ref{app:prec@k}.
Moreover, for zero-regret (i.e., optimal) node classifiers we obtain an optimal multi-label classifier in terms of estimation of marginal probabilities $\eta_j(\bx)$. This result can be further extended for precision@$k$. 

\begin{theorem} \label{thm:prec}
For any distribution $\prob$ and classifier $\bh$ delivering estimates $\heta_j(\bx)$ of the marginal probabilities of labels, the following holds:
\begin{align*}
\reg_{p@k} (\bh \given \bx) = \frac{1}{k}\sum_{i \in \calY_k} \eta_i(\bx) - \frac{1}{k}\sum_{j \in \hat \calY_k} \eta_j(\bx) \le 2 \max_{l} \left | \eta_l(\bx) - \heta_l(\bx) \right |
\end{align*}
\end{theorem}
The proof is based on adding and subtracting the following terms $\frac{1}{k}\sum_{i \in \calY_k} \heta_i(\bx)$ and $\frac{1}{k}\sum_{j \in \hat \calY_k} \heta_j(\bx)$ to the regret (a detailed proof is given in Appendix~\ref{app:prec@k}).
By getting together both theorems we get an upper bound of the precision@$k$ regret expressed in terms of the regret of the node classifiers. Again, for the zero-regret node classifiers, we get optimal solution in terms of precision@$k$.

\vspace{\sectionBefore}
\section{Implementation details of PLTs}
\label{sec:plt-tagging}
\vspace{\sectionAfter}


Given the tree structure, the node classifiers of \Algo{PLT}s can be trained as logistic regression either in online~\citep{Jasinska_et_al_2016} or batch mode~\citep{Prabhu_et_al_2018}. Both training modes have their pros and cons, but the online implementation gives a possibility of learning more complex representation of input instances. The above cited  implementations are both based on sparse representation, given either in a form of a bag-of-words or its TF-IDF variant. We opt here for training a \Algo{PLT} in the online mode along with the dense representation. We build our implementation  upon \Algo{fastText} and refer to it as \Algo{XT} which stands for \Algo{extremeText}.\footnote{Implementation of \Algo{XT} is available at \url{https://github.com/mwydmuch/extremeText}.}
%
In this way, we succeeded to obtain a very powerful and compressed model. The small dense models are important for fast online prediction as they do not need too much resources. The sparse models, in turn, can be slow and expensive in terms of memory usage as they need to decompress the node models to work fast. 
Remark also that, in general, \Algo{PLT}s can be used as an output layer of any neural network architecture (also that one used in XML-CNN~\citep{Yen_et_al_2017}) to speed up training and prediction time. 

In contrast to the original implementation of \Algo{fastText}, we use L2 regularization for all parameters of the model. 
To obtain representation of input instances we do not compute simple averages of the feature vectors, but use weights proportional to the TF-IDF scores of features. The competitive results can be obtained with feature and instance vectors of size 500. 
If a node classification task contains only positive instances, we use a constant classifier predicting 1 without any training. The training of \Algo{PLT} in either mode, online or batch, can be easily parallelized as each node classifier can be trained in isolation from the other classifiers. In our current implementation, however, we follow the parallelization on the level of training and test instances as in original \Algo{fastText}.

Our implementation, because of the additional use of the L2 regularization, has more parameters than original~\Algo{fastText}. We have found, however, 
that our model is remarkably robust for the hyperparameter selection, since it achieves close to optimal performance for a large set of hyperparameters that is in the vicinity of the optimal one. 
Moreover, the optimal hyperparameters are close to each other across all datasets. 
We report more information about the hyperparameter selection in Appendix~\ref{sec:hyper}.

The tree structure of a \Algo{PLT} is a crucial modeling decision. The vanishing regret for probability estimates and precision@$k$ holds regardless of the tree structure (see Theorem \ref{thm:log} and \ref{thm:prec}), however, this theory requires the regret of the node classifiers also to vanish. In practice, we can only estimate the conditional probabilities in the nodes, therefore the tree structure does indeed matter as it affects the difficulty of the node learning problems. The original \Algo{PLT} paper~\citep{Jasinska_et_al_2016} uses simple complete trees with labels assigned to leaves according to their frequencies. Another option, routinely used in \Algo{HSM}~\citep{Joulin_et_al_2016}, is the Huffman tree built over the label frequencies. Such tree takes into account the computational complexity by putting the most frequent labels close to the root. 
This approach has been further extended to optimize GPU operations in~\citep{Grave_et_al_2017}. Unfortunately, it ignores the statistical properties of the tree structure. Furthermore, for multi-label case the Huffman tree is no longer optimal even in terms of computational cost as we show it 
in Appendix~\ref{App:OptComp}. There exist, however, other methods that focus on building a tree with high overall accuracy~\citep{Tagami_2017,Prabhu_et_al_2018}. In our work, we follow the later approach, which performs a simple top-down hierarchical clustering. Each label in this approach is represented by a profile vector being an average of the training vectors tagged by this label. Then the profile vectors are clustered using balanced k-means which divides the labels into two or more clusters with approximately the same size. This procedure is then repeated recursively until the clusters are smaller than a given value (for example, 100). The nodes of the resulting tree are then of different arities. The internal nodes up to the leaves' parent nodes have $k$ children, but the leaves' parent nodes are usually of higher arity. Thanks to this clustering, similar labels are close to each other in the tree. Moreover, the tree is balanced, so its depth is logarithmic in terms of the number of labels.

\vspace{\sectionBefore}
\section{Empirical results}
\label{sec:empirical_results}
\vspace{\sectionAfter}


We carried out three sets of experiments. In the first, we compare exhaustively the performance of \Algo{PLT}s and \Algo{HSM} on synthetic and benchmark data. Due to lack of space, the results are deferred to Appendix \ref{sec:empirical-synthetic} and \ref{sec:empirical-benchmark}. The results on synthetic data confirm our theoretical findings: the models are the same in the case of multi-class data, the performance of \Algo{HSM} and \Algo{PLT}s is on par using multi-label data with independent labels, and \Algo{PLT}s significantly outperform \Algo{HSM} on multi-label data with conditionally dependent labels. The results on the benchmark data clearly indicate the better performance of \Algo{PLT}s over \Algo{HSM}.


In the second experiment, we compare \Algo{XT}, the variant of \Algo{PLT}s discussed in the previous section, to the state-of-the-art algorithms on five benchmark datasets taken from XMLC repository,\footnote{Additional statistics of these datasets are also included in Appendix~\ref{app:benchmark_datasets}. Address of the XMLC repository: \url{http://manikvarma.org/downloads/XC/XMLRepository.html}} and their text equivalents, by courtesy of \citet{Liu_et_al_2017}.
We compare the models in terms of precision@$\{1,3,5\}$, model size, training and test time. The competitors for our~\Algo{XT} are original~\Algo{fastText}, its variant~\Algo{Learned Tree}, a \Algo{PLT}-like batch learning algorithm \Algo{Parabel} (we use the variant that uses a single tree instead of an ensemble), a XMLC-designed convolutional deep network \Algo{XML-CNN}, a decision tree ensemble \Algo{FastXML}, and two 1-vs-All approaches tailored to XMLC problems, \Algo{PPD-Sparse} and \Algo{DiSMEC}. The hyperparameters of the models have been tuned using grid search. The range of the hyperparameters is reported in~\ref{sec:hyper}.  




The results presented in Table \ref{tab:datasets-differences-table} demonstrate that \Algo{XT} outperforms the \Algo{HSM} approaches with the pick-one-label heuristic, namely \Algo{fastText} and \Algo{Learned Tree}, with a large margin. This proves the superiority of \Algo{PLT}s as the proper generalization of \Algo{HSM} to multi-label setting. In all the above methods we use vectors of length 500 and we tune the other hyperparameters appropriately for a fair comparison.

Moreover, \Algo{XT} scales well to extreme datasets achieving performance close to the state-of-the-art, being at the same time 10000x and 100x faster compared to \Algo{DiSMEC} and \Algo{PPDSparse} during prediction. \Algo{XT} always responds below 2ms, what makes it a competitive alternative for an online setting. \Algo{XT} is also close to \Algo{Parabel} in terms of performance. However, the reported times and model sizes of \Algo{Parabel} are given for the batch prediction. The prediction times seem to be faster, but \Algo{Parabel} needs to decompress the model during prediction, what makes it less suitable for online prediction. It is only efficient when the batches are sufficiently large.  
Finally, we would like to underline that \Algo{XT} outperforms \Algo{XML-CNN}, the more complex neural network, in terms of predictive performance with computational costs that are an order of magnitude smaller. Moreover, \Algo{XML-CNN} requires pretrained embedding vectors, whereas \Algo{XT} can be used with random initialization.

\begin{table*}[t]
        \caption{\small Precision@$k$ scores with $k=\{ 1,3,5\}$ and statistics of \Algo{FastXML}, \Algo{PPDSparse}, \Algo{DiSMEC}, \Algo{Parabel} (with 1 tree), \Algo{fastText} (\Algo{FT}), \Algo{Learned Tree} (\Algo{LT}), \Algo{extremeText} (\Algo{XT}) and \Algo{XML-CNN} methods. Notation:
        $N$ -- number of samples, $T$ -- CPU time, $m$ -- number of labels, $d$ -- number of features, $\ast$ -- result of offline prediction, $\star$ -- calculated on GPU, $\dagger$ -- not reported by authors, $\ddagger$  -- cannot be calculated due to lack of a text version of a dataset. }
        \label{tab:datasets-differences-table}
        \begin{center}
                \begin{normalsize}
                        \resizebox{\textwidth}{!}{
                        \begin{tabular}{@{}l | l@{\hskip 0pt} | l@{\hskip 0pt} | l@{\hskip 0pt} | l@{\hskip 0pt} || l@{\hskip 0pt} | l@{\hskip 0pt} | l@{\hskip 0pt} || l@{\hskip 0pt} || l@{}}
                                \toprule
\multicolumn{1}{c|}{Dataset} 
& \multicolumn{1}{|c|}{Metrics} 
& \multicolumn{1}{|c|}{\Algo{FastXML}} 
& \multicolumn{1}{|c|}{\Algo{PPDSparse}}
& \multicolumn{1}{|c||}{\Algo{DiSMEC}}
& \multicolumn{1}{|c|}{\Algo{FT}} 
& \multicolumn{1}{|c|}{\Algo{LT}}
& \multicolumn{1}{|c||}{\textbf{\Algo{XT}}} 
& \multicolumn{1}{|c||}{\Algo{Parabel}}
& \multicolumn{1}{|c}{\Algo{XML-CNN}} \\
                                \midrule
\datasix{\textbf{Wiki-30K}}{$N_{train} = 14146$}{$N_{test} = 6616$}{$d = 101938$}{$m = 30938$}{}
& \datasix{P@1}{P@3}{P@5}{$T_{train}$}{$T_{test}/N_{test}$}{model size} 
& \datasix{82.03}{67.47}{57.76}{16m}{3.00ms}{354M}
& \datasix{73.80}{60.90}{50.40}{$\dagger$}{$\dagger$}{$\dagger$}
& \datasix{85.20}{\textbf{74.60}}{\textbf{65.90}}{$\dagger$}{$\dagger$}{$\dagger$}
& \datasix{80.78}{50.46}{36.79}{10m}{1.88ms}{513M}
& \datasix{80.85}{50.59}{37.68}{12m}{10.09ms}{513M}
& \datasix{\textbf{85.23}}{73.18}{63.39}{18m}{\textbf{0.83ms}}{\textbf{259M}}
& \datasix{83.77}{71.96}{62.44}{\textbf{5m}}{1.63ms$\ast$}{\textbf{109M$\ast$}}
& \datasix{82.78}{66.34}{56.23}{88m$\star$}{1.39ms$\star$}{$\star$} \\
\midrule

\datasix{\textbf{Delicious-200K}}{$N_{train} = 196606$}{$N_{test} = 100095$}{$d = 782585$} {$m = 205443$}{}
& \datasix{P@1}{P@3}{P@5}{$T_{train}$}{$T_{test}/N_{test}$}{model size} 
& \datasix{42.81}{38.76}{36.34}{458m}{4.86ms}{15.4G} 
& \datasix{45.05}{38.34}{34.90}{4781m}{275ms}{9.4G} 
& \datasix{44.71}{38.08}{34.7}{1080h}{5m}{18.0G} 
& \datasix{42.22}{37.90}{35.05}{271m}{1.97ms}{9.0G}
& \datasix{42.71}{36.27}{33.43}{563m}{1.98ms}{9.0G}
& \datasix{\textbf{47.85}}{\textbf{42.08}}{\textbf{39.13}}{502m}{\textbf{1.41ms}} {\textbf{1.9G}}
& \datasix{43.32}{38.49}{35.83}{105m}{\textbf{1.31ms$\ast$}}{\textbf{1.8G$\ast$}} 
& \datasix{$\ddagger$}{$\ddagger$}{$\ddagger$}{$\ddagger$}{$\ddagger$}{$\ddagger$} \\
\midrule

\datasix{\textbf{WikiLSHTC}}{$N_{train} = 1778351$}{$N_{test} = 587084$}{$d = 617899$} {$m = 325056$}{} 
& \datasix{P@1}{P@3}{P@5}{$T_{train}$}{$T_{test}/N_{test}$}{model size} 
& \datasix{49.35}{32.69}{24.03}{724m}{2.17ms}{9.3G}
& \datasix{64.08}{41.26}{30.12}{236m}{37.76ms}{5.2G}
& \datasix{\textbf{64.94}}{\textbf{42.71}}{\textbf{31.5}}{750h}{43m}{3.8G}
& \datasix{41.13}{24.09}{17.44}{207}{1.25ms}{6.5G}
& \datasix{50.15}{31.95}{23.59}{212m}{4.76ms}{6.5G} 
& \datasix{58.73}{39.24}{29.26}{550m}{\textbf{0.81ms}}{\textbf{3.3G}} 
& \datasix{61.53}{40.07}{29.25}{\textbf{34m}}{0.92ms$\ast$}{\textbf{1.1G$\ast$}}
& \datasix{$\ddagger$}{$\ddagger$}{$\ddagger$}{$\ddagger$}{$\ddagger$}{$\ddagger$} \\
\midrule

\datasix{\textbf{Wiki-500K}}{$N_{train} = 1813391$}{$N_{test} = 783743$}{$d = 2381304$} {$m = 501070$}{} 
& \datasix{P@1}{P@3}{P@5}{$T_{train}$}{$T_{test}/N_{test}$}{model size} 
& \datasix{54.10}{29.45}{21.21}{3214m}{8.03ms}{63G}
& \datasix{70.16}{50.57}{39.66}{1771m}{113.70ms}{3.4G}
& \datasix{\textbf{70.20}}{\textbf{50.60}}{\textbf{39.70}}{7495h}{155m}{14.7G}
& \datasix{32.73}{19.02}{14.46}{496m}{2.05ms}{11G}
& \datasix{37.18}{21.62}{16.01}{531m}{6.43ms}{11G}
& \datasix{64.48}{45.84}{35.46}{1253m}{\textbf{1.07ms}}{\textbf{5.5G}} 
& \datasix{66.12}{47.02}{36.45}{\textbf{168m}}{4.68ms$\ast$}{\textbf{2.0G$\ast$}}
& \datasix{59.85}{39.28}{29.81}{7032m$\star$}{21.06ms$\star$}{\textbf{3.7G$\star$}} \\
\midrule

\datasix{\textbf{Amazon-670K}}{$N_{train} = 490449$}{$N_{test} = 153025$}{$d = 135909$} {$m = 670091$}{} 
& \datasix{P@1}{P@3}{P@5}{$T_{train}$}{$T_{test}/N_{test}$}{model size} 
& \datasix{34.24}{29.30}{26.12}{422m}{3.39ms}{10G}
& \datasix{45.32}{40.37}{36.92}{102m}{66.09ms}{6.0G} 
& \datasix{\textbf{45.37}}{\textbf{40.40}}{\textbf{36.96}}{373h}{23m}{3.8G}
& \datasix{25.47}{21.47}{18.61}{162m}{7.84ms}{3.2G}
& \datasix{27.67}{20.96}{17.72}{182m}{5.13ms}{3.2G}
& \datasix{39.90}{35.36}{32.04}{241m}{\textbf{1.72ms}}{\textbf{1.5G}}
& \datasix{41.59}{37.18}{33.85}{\textbf{8m}}{\textbf{0.68ms$*$}}{\textbf{0.7G$*$}}
& \datasix{35.39}{33.74}{32.64}{3134m$\star$}{16.18ms$\star$}{\textbf{1.5G$\star$}} \\
                                \bottomrule
                        \end{tabular}}
                \end{normalsize}
        \end{center}
\end{table*}

In the third experiment we perform an ablation analysis in which we compare different components of the \Algo{XT} algorithm. We analyze the influence of the Huffman tree vs. top-down clustering, the simple averaging of features vectors vs. the TF-IDF-based weighting, and no regularization vs. L2 regularization. Figure~\ref{fig:ab} clearly shows that the components need to be combined together to obtain the best results. The best combination uses top-down clustering, TF-IDF-based weighting, and L2 regularization, while top-down clustering alone gets worse results than Huffman trees with TF-IDF-based weighting and L2 regularization. In Appendix~\ref{app:ablation-analysis} we give more detailed results of the ablation analysis performed on a larger spectrum of benchmark datasets. %

\definecolor{yahooviolet}{RGB}{66, 2, 176}
\definecolor{putblue}{RGB}{0, 103, 144}

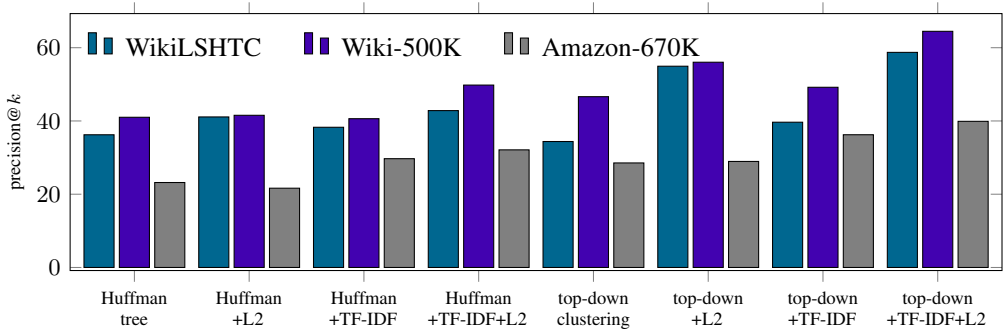
\begin{figure}[h]
\begin{center}
\scalebox{1.0}{
\begin{tikzpicture}
\begin{axis}[
    ybar,
    enlargelimits=0.08,
    bar width=4mm,
    ylabel={precision@$k$},
    ymin=4,
    symbolic x coords={Huffman\\tree, Huffman\\+L2, Huffman\\+TF-IDF, Huffman\\+TF-IDF+L2, 
    top-down\\clustering, top-down\\+L2, top-down\\+TF-IDF, top-down\\+TF-IDF+L2},
    xtick=data,
    width=\textwidth,
    height=5cm,
    x tick label style={font=\scriptsize, text width=4cm, align=center, rotate=0},
    y tick label style={font=\footnotesize},
    y label style={font=\scriptsize, at={(axis description cs:0.06,.5)}, anchor=south},
    legend style={
        align=left,
        at={(0.01,0.95)},
        anchor=north west,
        draw=none, 
        text depth=0pt,
        fill=none,
        legend columns=-1,
        column sep=1cm,
        /tikz/column 2/.style={column sep=0pt},
        /tikz/every odd column/.append style={column sep=0.1cm},
        /tikz/every even column/.append style={column sep=0.5cm}
    },
    ]
    
\addplot[fill=putblue] coordinates {(Huffman\\tree, 36.23)(Huffman\\+L2,41.10)(Huffman\\+TF-IDF, 38.28)(Huffman\\+TF-IDF+L2, 42.83)(top-down\\clustering, 34.39)(top-down\\+L2, 54.95)(top-down\\+TF-IDF, 39.66)(top-down\\+TF-IDF+L2, 58.73)}; 

\addplot[fill=yahooviolet] coordinates {(Huffman\\tree, 41.01)(Huffman\\+L2, 41.55)(Huffman\\+TF-IDF, 40.62)(Huffman\\+TF-IDF+L2, 49.80)(top-down\\clustering, 46.63)(top-down\\+L2, 56.04)(top-down\\+TF-IDF, 49.20)(top-down\\+TF-IDF+L2, 64.48)}; 

\addplot[fill=gray] coordinates {(Huffman\\tree, 23.194)(Huffman\\+L2, 21.642)(Huffman\\+TF-IDF, 29.70)(Huffman\\+TF-IDF+L2, 32.111)(top-down\\clustering, 28.538)(top-down\\+L2, 28.95)(top-down\\+TF-IDF, 36.239)(top-down\\+TF-IDF+L2, 39.90)}; 

\legend{WikiLSHTC, Wiki-500K, Amazon-670K}
\end{axis}
\end{tikzpicture}
}
\end{center}
\caption{The ablation analysis of different variants of XT on \textsc{WikiLSHTC}, \textsc{Wiki-500k}, and \textsc{Amazon-670K}.}
\label{fig:ab}
\end{figure}

\section{Conclusions}

In this paper we have proven that probabilistic label trees (\Algo{PLT}s) are no-regret generalization of \Algo{HSM} to the multi-label setting. Our main theoretical contribution is the precision@$k$ regret bound for \Algo{PLT}s. Moreover, we have shown that the pick-one-label heuristic commonly-used with \Algo{HSM} in multi-label problems leads to inconsistent results in terms of marginal probability estimation and precision@$k$. Our implementation of \Algo{PLT}s referred to as \Algo{XT}, built upon \Algo{fastText}, gets state-of-the-art results, being significantly better than the original \Algo{fastText}, \Algo{Learned Tree}, and \Algo{XML-CNN}. The \Algo{XT} results are also close to the best known ones that are obtained by expensive 1-vs-All approaches, such as \Algo{PPDSparse} and \Algo{DiSMEC}, and outperforms the other tree-based methods on many benchmarks.
Our online variant has the advantage of producing very often much smaller models that can be efficiently used in fast online prediction. 


\section*{Acknowledgements}

The work of Kalina Jasinska was supported by the Polish National Science Center under grant no. 2017/25/N/ST6/00747.
The work of Krzysztof Dembczyński was supported by the Polish Ministry of Science and Higher Education under grant no. 09/91/DSPB/0651.
\newpage

\bibliography{xmlc_references}

\begin{thebibliography}{38}
\providecommand{\natexlab}[1]{#1}
\providecommand{\url}[1]{\texttt{#1}}
\expandafter\ifx\csname urlstyle\endcsname\relax
  \providecommand{\doi}[1]{doi: #1}\else
  \providecommand{\doi}{doi: \begingroup \urlstyle{rm}\Url}\fi

\bibitem[Agarwal(2014)]{Agarwal_2014}
Agarwal, S.
\newblock Surrogate regret bounds for bipartite ranking via strongly proper
  losses.
\newblock \emph{JMLR}, 15:\penalty0 1653--1674, 2014.

\bibitem[Agrawal et~al.(2013)Agrawal, Gupta, Prabhu, and
  Varma]{Agrawal_et_al_2013}
Agrawal, R., Gupta, A., Prabhu, Y., and Varma, M.
\newblock Multi-label learning with millions of labels: Recommending advertiser
  bid phrases for web pages.
\newblock In \emph{WWW}, pp.\  13--24. ACM, 2013.

\bibitem[Babbar \& Sch\"{o}lkopf(2017)Babbar and
  Sch\"{o}lkopf]{Babbar_Scholkopf_2017}
Babbar, Rohit and Sch\"{o}lkopf, Bernhard.
\newblock Dismec: Distributed sparse machines for extreme multi-label
  classification.
\newblock In \emph{WSDM 2017}, pp.\  721--729. ACM, 2017.

\bibitem[Bengio et~al.(2010)Bengio, Weston, and Grangier]{Bengio_et_al_2010}
Bengio, S., Weston, J., and Grangier, D.
\newblock Label embedding trees for large multi-class tasks.
\newblock In \emph{{NIPS} 24}, pp.\  163--171, 2010.

\bibitem[Beygelzimer et~al.(2009{\natexlab{a}})Beygelzimer, Langford, Lifshits,
  Sorkin, and Strehl]{Beygelzimer_et_al_2009b}
Beygelzimer, A., Langford, J., Lifshits, Y., Sorkin, G., and Strehl, A.
\newblock Conditional probability tree estimation analysis and algorithms.
\newblock In \emph{UAI}, pp.\  51--58, 2009{\natexlab{a}}.

\bibitem[Beygelzimer et~al.(2009{\natexlab{b}})Beygelzimer, Langford, and
  Ravikumar]{Beygelzimer_et_al_2009a}
Beygelzimer, A., Langford, J., and Ravikumar, P.
\newblock Error-correcting tournaments.
\newblock In \emph{ALT}, pp.\  247--262. Springer, 2009{\natexlab{b}}.

\bibitem[Bhatia et~al.(2015)Bhatia, Jain, Kar, Varma, and
  Jain]{Bhatia_et_al_2015}
Bhatia, K., Jain, H., Kar, P., Varma, M., and Jain, P.
\newblock Sparse local embeddings for extreme multi-label classification.
\newblock In \emph{{NIPS} 29}, pp.\  730--738, 2015.

\bibitem[Choromanska \& Langford(2015)Choromanska and
  Langford]{Choromanska_Langford_2015}
Choromanska, A. and Langford, J.
\newblock Logarithmic time online multiclass prediction.
\newblock In \emph{{NIPS} 29}, pp.\  55--63, 2015.

\bibitem[Dekel \& Shamir(2010)Dekel and Shamir]{Dekel_Shamir_2010}
Dekel, O. and Shamir, O.
\newblock Multiclass-multilabel classification with more classes than examples.
\newblock In \emph{AISTATS}, pp.\  137--144, 2010.

\bibitem[Dembczy{\'n}ski et~al.(2010)Dembczy{\'n}ski, Cheng, and
  H\"ullermeier]{Dembczynski_et_al_2010c}
Dembczy{\'n}ski, K., Cheng, W., and H\"ullermeier, E.
\newblock Bayes optimal multilabel classification via probabilistic classifier
  chains.
\newblock In \emph{ICML}, pp.\  279--286, 2010.

\bibitem[Dembczy{\'n}ski et~al.(2012)Dembczy{\'n}ski, Waegeman, Cheng, and
  H\"ullermeier]{Dembczynski_et_al_2012c}
Dembczy{\'n}ski, K., Waegeman, W., Cheng, W., and H\"ullermeier, E.
\newblock An analysis of chaining in multi-label classification.
\newblock In \emph{ECAI}, 2012.

\bibitem[Dembczy\'nski et~al.(2016)Dembczy\'nski, Kot{\l}owski, Waegeman,
  Busa-Fekete, and H\"ullermeier]{Dembczynski_et_al_2016}
Dembczy\'nski, Krzysztof, Kot{\l}owski, Wojciech, Waegeman, Willem,
  Busa-Fekete, R\'obert, and H\"ullermeier, Eyke.
\newblock Consistency of probabilistic classifier trees.
\newblock In \emph{ECMLPKDD}. Springer, 2016.

\bibitem[Dembczy{\'{n}}ski et~al.(2017)Dembczy{\'{n}}ski, Kot{\l}owski, Koyejo,
  and Natarajan]{Dembczynski_et_al_2017}
Dembczy{\'{n}}ski, Krzysztof, Kot{\l}owski, Wojciech, Koyejo, Oluwasanmi, and
  Natarajan, Nagarajan.
\newblock Consistency analysis for binary classification revisited.
\newblock In Precup, Doina and Teh, Yee~Whye (eds.), \emph{Proceedings of the
  34th International Conference on Machine Learning}, volume~70 of
  \emph{Proceedings of Machine Learning Research}, pp.\  961--969. PMLR, 2017.

\bibitem[Deng et~al.(2011)Deng, Satheesh, Berg, and Fei-Fei]{Deng_et_al_2011}
Deng, J., Satheesh, S., Berg, A.~C., and Fei-Fei, L.
\newblock Fast and balanced: Efficient label tree learning for large scale
  object recognition.
\newblock In \emph{{NIPS} 24}, pp.\  567--575, 2011.

\bibitem[Fan et~al.(2008)Fan, Chang, Hsieh, Wang, and Lin]{liblinear}
Fan, R.-E., Chang, K.-W., Hsieh, C.-J., Wang, X.-R., and Lin, C.-J.
\newblock {LIBLINEAR}: A library for large linear classification.
\newblock \emph{JMLR}, 9:\penalty0 1871--1874, 2008.

\bibitem[Fox(1997)]{Fox_1997}
Fox, J.
\newblock \emph{Applied regression analysis, linear models, and related
  methods}.
\newblock Sage, 1997.

\bibitem[Grave et~al.(2017)Grave, Joulin, Ciss{\'{e}}, Grangier, and
  J{\'{e}}gou]{Grave_et_al_2017}
Grave, Edouard, Joulin, Armand, Ciss{\'{e}}, Moustapha, Grangier, David, and
  J{\'{e}}gou, Herv{\'{e}}.
\newblock Efficient softmax approximation for {GPU}s.
\newblock In Precup, Doina and Teh, Yee~Whye (eds.), \emph{Proceedings of the
  34th International Conference on Machine Learning, {ICML} 2017}, volume~70 of
  \emph{Proceedings of Machine Learning Research}, pp.\  1302--1310,
  International Convention Centre, Sydney, Australia, 2017. PMLR.

\bibitem[Jasinska et~al.(2016)Jasinska, Dembczynski, Busa-Fekete, Pfannschmidt,
  Klerx, and H{\"{u}}llermeier]{Jasinska_et_al_2016}
Jasinska, K., Dembczynski, K., Busa-Fekete, R., Pfannschmidt, K., Klerx, T.,
  and H{\"{u}}llermeier, E.
\newblock Extreme {F}-measure maximization using sparse probability estimates.
\newblock In \emph{ICML}, 2016.

\bibitem[Jernite et~al.(2017)Jernite, Choromanska, and
  Sontag]{Jernite_et_al_2017}
Jernite, Yacine, Choromanska, Anna, and Sontag, David.
\newblock Simultaneous learning of trees and representations for extreme
  classification and density estimation.
\newblock In Precup, Doina and Teh, Yee~Whye (eds.), \emph{Proceedings of the
  34th International Conference on Machine Learning}, volume~70 of
  \emph{Proceedings of Machine Learning Research}, pp.\  1665--1674,
  International Convention Centre, Sydney, Australia, 2017. PMLR.

\bibitem[Joulin et~al.(2016)Joulin, Grave, Bojanowski, and
  Mikolov]{Joulin_et_al_2016}
Joulin, Armand, Grave, Edouard, Bojanowski, Piotr, and Mikolov, Tomas.
\newblock Bag of tricks for efficient text classification.
\newblock \emph{CoRR}, abs/1607.01759, 2016.
\newblock URL \url{http://arxiv.org/abs/1607.01759}.

\bibitem[Kumar et~al.(2013)Kumar, Vembu, Menon, and Elkan]{Kumar_et_al_2013}
Kumar, A., Vembu, S., Menon, A.K., and Elkan, C.
\newblock Beam search algorithms for multilabel learning.
\newblock In \emph{Machine Learning}, 2013.

\bibitem[Kurzynski(1988)]{Kurzynski_1988}
Kurzynski, Marek.
\newblock On the multistage bayes classifier.
\newblock \emph{Pattern Recognition}, 21\penalty0 (4):\penalty0 355--365, 1988.

\bibitem[Langford et~al.(2007)Langford, Strehl, and Li]{Langford_et_al_2007}
Langford, J., Strehl, A., and Li, L.
\newblock Vowpal wabbit, 2007.
\newblock \url{http://mloss.org/software/view/53/}.

\bibitem[Li \& Lin(2014)Li and Lin]{Li_Lin_2014}
Li, Ch.-L. and Lin, H.-Ti.
\newblock Condensed filter tree for cost-sensitive multi-label classification.
\newblock In \emph{ICML}, pp.\  423--431, 2014.

\bibitem[Liu et~al.(2017)Liu, Chang, Wu, and Yang]{Liu_et_al_2017}
Liu, Jingzhou, Chang, Wei-Cheng, Wu, Yuexin, and Yang, Yiming.
\newblock Deep learning for extreme multi-label text classification.
\newblock In \emph{SIGIR 2017}, pp.\  115--124. ACM, 2017.

\bibitem[Mikolov et~al.(2013)Mikolov, Sutskever, Chen, Corrado, and
  Dean]{Mikolov_et_al_2013}
Mikolov, Tomas, Sutskever, Ilya, Chen, Kai, Corrado, Greg~S, and Dean, Jeff.
\newblock Distributed representations of words and phrases and their
  compositionality.
\newblock In \emph{{NIPS} 26}, pp.\  3111--3119. Curran Associates, Inc., 2013.

\bibitem[Mineiro \& Karampatziakis(2015)Mineiro and
  Karampatziakis]{Mineiro_Karampatziakis_2015}
Mineiro, P. and Karampatziakis, N.
\newblock Fast label embeddings via randomized linear algebra.
\newblock In \emph{ECML/PKDD 2015}, pp.\  37--51, 2015.

\bibitem[Morin \& Bengio(2005)Morin and Bengio]{Morin_Bengio_2005}
Morin, F. and Bengio, Y.
\newblock Hierarchical probabilistic neural network language model.
\newblock In \emph{AISTATS}, pp.\  246--252, 2005.

\bibitem[Narasimhan et~al.(2014)Narasimhan, Vaish, and
  S.]{Narasimhan_et_al_2014}
Narasimhan, H., Vaish, R., and S., Agarwal.
\newblock On the statistical consistency of plug-in classifiers for
  non-decomposable performance measures.
\newblock In \emph{{NIPS} 27}, pp.\  1493--1501, 2014.

\bibitem[Niculescu-Mizil \& Abbasnejad(2017)Niculescu-Mizil and
  Abbasnejad]{Niculescu-Mizil_Abbasnejad_2017}
Niculescu-Mizil, Alexandru and Abbasnejad, Ehsan.
\newblock {Label Filters for Large Scale Multilabel Classification}.
\newblock In Singh, Aarti and Zhu, Jerry (eds.), \emph{Proceedings of the 20th
  International Conference on Artificial Intelligence and Statistics},
  volume~54 of \emph{Proceedings of Machine Learning Research}, pp.\
  1448--1457. PMLR, 2017.

\bibitem[Prabhu \& Varma(2014)Prabhu and Varma]{Prabhu_Varma_2014}
Prabhu, Y. and Varma, M.
\newblock {FastXML}: A fast, accurate and stable tree-classifier for extreme
  multi-label learning.
\newblock In \emph{KDD}, pp.\  263--272. ACM, 2014.

\bibitem[Prabhu et~al.(2018)Prabhu, Kag, Harsola, Agrawal, and
  Varma]{Prabhu_et_al_2018}
Prabhu, Yashoteja, Kag, Anil, Harsola, Shrutendra, Agrawal, Rahul, and Varma,
  Manik.
\newblock Parabel: Partitioned label trees for extreme classification with
  application to dynamic search advertising.
\newblock In Champin, Pierre{-}Antoine, Gandon, Fabien~L., Lalmas, Mounia, and
  Ipeirotis, Panagiotis~G. (eds.), \emph{Proceedings of the 2018 World Wide Web
  Conference on World Wide Web, {WWW} 2018}, pp.\  993--1002. {ACM}, 2018.

\bibitem[Shrivastava \& Li(2015)Shrivastava and Li]{Shrivastava_Li_2015}
Shrivastava, A. and Li, P.
\newblock Improved asymmetric locality sensitive hashing ({ALSH}) for maximum
  inner product search (mips).
\newblock In \emph{UAI}, 2015.

\bibitem[Tagami(2017)]{Tagami_2017}
Tagami, Yukihiro.
\newblock Annexml: Approximate nearest neighbor search for extreme multi-label
  classification.
\newblock In \emph{Proceedings of the 23rd ACM SIGKDD International Conference
  on Knowledge Discovery and Data Mining}, pp.\  455--464, New York, NY, USA,
  2017. ACM.

\bibitem[Vijayanarasimhan et~al.(2014)Vijayanarasimhan, Shlens, Monga, and
  Yagnik]{Vijayanarasimhan_et_al_2014}
Vijayanarasimhan, Sudheendra, Shlens, Jonathon, Monga, Rajat, and Yagnik, Jay.
\newblock Deep networks with large output spaces.
\newblock \emph{CoRR}, abs/1412.7479, 2014.
\newblock URL \url{http://arxiv.org/abs/1412.7479}.

\bibitem[Ye et~al.(2012)Ye, Chai, Lee, and Chieu]{Ye_et_al_2012}
Ye, N., Chai, A., Lee, W., and Chieu, H.
\newblock Optimizing {F}-measures: {A} tale of two approaches.
\newblock In \emph{ICML}, 2012.

\bibitem[Yen et~al.(2017)Yen, Huang, Dai, Ravikumar, Dhillon, and
  Xing]{Yen_et_al_2017}
Yen, Ian~E.H., Huang, Xiangru, Dai, Wei, Ravikumar, Pradeep, Dhillon, Inderjit,
  and Xing, Eric.
\newblock {PPDSparse}: A parallel primal-dual sparse method for extreme
  classification.
\newblock In \emph{KDD 2017}, pp.\  545--553. ACM, 2017.

\bibitem[Yu et~al.(2014)Yu, Jain, Kar, and Dhillon]{Yu_et_al_2014}
Yu, Hsiang-Fu, Jain, Prateek, Kar, Purushottam, and Dhillon, Inderjit.
\newblock Large-scale multi-label learning with missing labels.
\newblock In \emph{ICML 2014}, volume~32, pp.\  593--601. PMLR, 2014.

\end{thebibliography}
\bibliographystyle{icml2018}

\appendix

\onecolumn

\setcounter{theorem}{0}

\section{Regret for precision@$k$}
\label{app:prec@k}

Precision@$k$ is formally defined as:
$$
\mathrm{precision}@k(\by, \bx, \bh) = \frac{1}{k} \sum_{j \in \hat \calY_k} \assert{y_j = 1},
$$
where $\hat \calY_k$ is a set of $k$ labels predicted by $\bh$ for $\bx$.
The loss function for precision@$k$ can be defined as $\ell_{p@k} = 1 - \mathrm{precision}@k$. The conditional risk is then:
\begin{eqnarray*}
\loss_{p@k}(\bh \given \bx) & = & \mathbb{E}_{\by} \ell_{p@k}(\by,\bx, \bh) \\
& = & 1 - \sum_{\by \in \calY} \prob(\by \given \bx) \frac{1}{k} \sum_{j \in \hat \calY_k} \assert{y_j = 1} \\
& = & 1 - \frac{1}{k} \sum_{j \in \hat \calY_k} \sum_{\by \in \calY} \prob(\by \given \bx) \assert{y_j = 1} \\
& = & 1 - \frac{1}{k} \sum_{j \in \hat \calY_k} \eta_j(\bx) \,.
\end{eqnarray*}
The above result shows that the optimal strategy for precision@$k$ is to predict $k$ labels
with the highest marginal probabilities $\eta_j(\bx)$.

We show now that \Algo{PLT}s obey strong theoretical guarantees. We first recall the result from~\citep{Jasinska_et_al_2016} that relates the absolute difference between the true and the estimated marginal probability of label $j$, $|\eta_j(\bx) - \heta_j(\bx)|$, to the surrogate loss $\ell$ used to train node classifiers $f_{\bz^i}$. It is assumed here that $\ell$ is a strongly proper composite loss (e.g, logistic, exponential, or squared-error loss) characterized by a constant $\lambda$ (e.g. $\lambda = 4$ for logistic loss).\footnote{For more detailed introduction to strongly proper composite losses, we refer the reader to~\citep{Agarwal_2014}}

\begin{theorem}
For any distribution $\prob$ and internal node classifiers $f_{\bz^i}$, the following holds:
$$
\left | \eta_j(\bx) - \heta_j(\bx) \right |  \le  \sum_{i=0}^l \prob(\bz^{i-1} \given \bx) \sqrt{ \frac{2}{\lambda}} \sqrt{\reg_\ell(f_{\bz^i} \given \bz^{i-1}, \bx)} \,,
$$
where $\reg_\ell(f_{\bz^i} \given \bz^{i-1}, \bx)$ is a binary classification regret for a strongly proper composite loss $\ell$ and $\lambda$ is a constant specific for loss $\ell$.
\end{theorem}
\begin{proof}
From Equation~(\ref{eqn:plt}) we have:
$$
\eta_j(\bx) = \prob(\bz \given \bx) = \prod_{i = 0}^l \prob(z_i \given \bz^{i-1}, \bx) = \prob(\bz^{n-1} \given \bx)\prod_{i = n}^l \prob(z_i \given \bz^{i-1}, \bx)\,,
$$ 
for any $1 \le n \le l$. A similar equation holds for the estimates $\heta_j(\bx)$, $\hprob(z_i \given \bz^{i-1}, \bx)$, and $\hprob(\bz^{n-1} \given \bx)$. 
 
By expressing $\eta_j(\bx)$  and $\heta_j(\bx)$ in the aforementioned way we get:
\begin{eqnarray*}
\left | \eta_j(\bx) - \heta_j(\bx) \right |  & = & \Big| \prod_{i = 0}^l \prob(z_i \given \bz^{i-1}, \bx) - \prod_{i = 0}^l \hprob(z_i \given \bz^{i-1}, \bx) \Big| \\
& = & \Big| \prob(\bz^{l-1} \given \bx) \prob(z_l \given \bz^{l-1}, \bx) - \hprob(\bz^{l-1} \given \bx) \hprob(z_l \given \bz^{l-1}, \bx)\Big|
\end{eqnarray*}

By adding and subtracting $\hprob(z_l \given \bz^{l-1}, \bx)\prob(\bz^{l-1} \given \bx)$ and using the fact that $|a - b| \leq |a - c| + |b - c|$, and that probability values are in $[0 ,1]$, we can write:
\begin{eqnarray*}
\left | \eta_j(\bx) - \heta_j(\bx) \right | & = & \Big| \prob(\bz^{l} \given \bx) - \hprob(\bz^{l} \given \bx) \Big|\\
& = & \Big| \prob(\bz^{l-1} \given \bx) \prob(z_l \given \bz^{l-1}, \bx) - \hprob(\bz^{l-1} \given \bx) \hprob(z_l \given \bz^{l-1}, \bx)\Big| \\
& = & \Big| \prob(\bz^{l-1} \given \bx) \prob(z_l \given \bz^{l-1}, \bx) - \hprob(\bz^{l-1} \given \bx) \hprob(z_l \given \bz^{l-1}, \bx) \\
&  & \quad+ \quad \hprob(z_l \given \bz^{l-1}, \bx)\prob(\bz^{l-1} \given \bx) - \hprob(z_l \given \bz^{l-1}, \bx)\prob(\bz^{l-1} \given \bx)\Big| \\
& \leq & 
\Big|  \prob(\bz^{l-1} \given \bx) \prob(z_l \given \bz^{l-1}, \bx) - \hprob(z_l \given \bz^{l-1}, \bx)\prob(\bz^{l-1} \given \bx) \Big | \\
&  & \quad+ \quad \Big| \hprob(\bz^{l-1} \given \bx) \hprob(z_l \given \bz^{l-1}, \bx) - \hprob(z_l \given \bz^{l-1}, \bx)\prob(\bz^{l-1} \given \bx) \Big| \\ 
& = &  \prob(\bz^{l-1} \given \bx) \Big|  \prob(z_l \given \bz^{l-1}, \bx) - \hprob(z_l \given \bz^{l-1}, \bx) \Big | \\
&  &  \quad+ \quad \hprob(z_l \given \bz^{l-1}, \bx) \Big| \hprob(\bz^{l-1} \given \bx) - \prob(\bz^{l-1} \given \bx) \Big| \\
& \leq & 
 \prob(\bz^{l-1} \given \bx) \Big|  \prob(z_l \given \bz^{l-1}, \bx) - \hprob(z_l \given \bz^{l-1}, \bx) \Big | \\
 &  & \quad+ \quad \Big| \prob(\bz^{l-1}\given \bx) - \hprob(\bz^{l-1}\given \bx) \Big|
\end{eqnarray*}

\noindent
We notice that rightmost term corresponds to the absolute value of the difference of probabilities corresponding to one-element shorter code $\bz^{l-1}$. Therefore we can use recursion and write:
\begin{equation}
\left | \eta_j(\bx) - \heta_j(\bx) \right |  \leq  \sum_{i=0}^{l}\prob(\bz^{i-1}\given \bx) \Big| \prob(z_i | \bz^{i-1}, \bx) - \hprob(z_i | \bz^{i-1}, \bx) \Big|.
\label{eqn:estimation_bound}
\end{equation}

Next, we express the above bound in terms of the regret of the strongly proper composite losses. 
The $(\bx, z_{i})$ pairs are generated i.i.d. according to $\prob(\bx, z_i \given \bz^{i-1})$.
Assume that a node classifier has a form of a real-valued function $f_{\bz^i}$.
Moreover, there exists a strictly increasing (and therefore invertible) link function
$\psi: [0,1] \rightarrow \mathbb{R}$ such that $f_{\bz^i}(\bx) = \psi(\prob(z_i \given \bz^{i-1}, \bx))$. 
Recall that the regret of $f_{\bz^i}$ in terms of a loss function $\ell$ at point $\bx$ is defined as:
$$
\reg_{\ell}(f_{\bz^i} \given \bz^{i-1}, \bx) = L_{\ell}(f_{\bz^i} \given \bz^{i-1}, \bx) - L_{\ell}^*( \bz^{i-1},\bx) \,,
$$
where $L_{\ell}(f_{\bz^i} \given \bz^{i-1}, \bx)$ is the expected loss at point $\bx$:
$$
L_{\ell}(f_{\bz^i} \given \bz^{i-1}, \bx) =  \prob(z_i \given \bz^{i-1}, \bx) \ell  (1, f_{\bz^i}(\bx))   + (1 - \prob(z_i \given \bz^{i-1}, \bx)) \ell  (-1, f_{\bz^i}(\bx))  \,,
$$
and  $L_{\ell}^*(\bx)$ is the minimum expected loss at point $\bx$.

If a node classifier is trained by a learning algorithm that minimizes a strongly proper composite loss, then the bound (\ref{eqn:estimation_bound}) can be expressed in terms of the regret of this loss function~\citep{Agarwal_2014}: 
$$
\left | \prob(z_i \given \bz^{i-1}, \bx)  - \psi^{-1}(f_{\bz^i})  \right | \le \sqrt{ \frac{2}{\lambda}} \sqrt{\reg_\ell(f_{\bz^i} \given \bz^{i-1}, \bx)} \,.
$$
By putting the above inequality into~(\ref{eqn:estimation_bound}), we get
\begin{align*}
\left | \eta(\bx, j) - \heta(\bx, j) \right | 
& 
\le \! \sum_{i=0}^l  \prob(\bz^{i-1}\given \bx)\left | \prob(z_i \given \bz^{i-1}, \bx)  - \hprob(z_i \given \bz^{i-1}, \bx)  \right | \\ 
& 
= \!  \sum_{i=0}^l  \prob(\bz^{i-1} \given \bx) \left | \prob(z_i \given \bz^{i-1}, \bx) - \psi^{-1}(f_{\bz^{i}})  \right | \\ 
&
\le  \! \sum_{i=0}^l \prob(\bz^{i-1} \given \bx) \sqrt{ \frac{2}{\lambda}} \sqrt{\reg_\ell(f_{\bz^{i}} \given \bz^{i-1}, \bx)}
\end{align*} 
\end{proof}


The above result shows that the absolute error of estimating the marginal probability of label $j$ can be upperbounded by the regret of the node classifiers on the corresponding path from the root to a leaf. Moreover, for zero-regret (i.e., optimal) node classifiers we obtain an optimal multi-label classifier in terms of estimation of marginal probabilities $\eta_j(\bx)$. This result can be further extended for precision@$k$. 

Let us denote a set of the top $k$ labels with respect to the true marginals by $\calY_k$ and a set of the top k labels with respect to predicted marginals by $\hat \calY_k$.
The conditional regret for precision@$k$ is given then by:
$$
\reg_{p@k} (\bh \given \bx) = \frac{1}{k}\sum_{i \in \calY_k} \eta_i(\bx) - \frac{1}{k}\sum_{j \in \hat \calY_k} \eta_j(\bx)
$$

\begin{theorem}
For any distribution $\prob$ and classifier $\bh$ delivering estimates $\heta_j(\bx)$ of the marginal probabilities of labels, the following holds:
\begin{align*}
\reg_{p@k} (\bh \given \bx) = \frac{1}{k}\sum_{i \in \calY_k} \eta_i(\bx) - \frac{1}{k}\sum_{j \in \hat \calY_k} \eta_j(\bx) \le 2 \max_{l} \left | \eta_l(\bx) - \heta_l(\bx) \right |
\end{align*}
\end{theorem}
\begin{proof}
Let us add and subtract the following two terms, $\frac{1}{k}\sum_{i \in \calY_k} \heta_i(\bx)$ and $\frac{1}{k}\sum_{j \in \hat \calY_k} \heta_j(\bx)$, to the regret and reorganize the expression in the following way:
\begin{align*}
\reg_{p@k} (\bh \given \bx) & = \underbrace{\frac{1}{k}\sum_{i \in \calY_k} \eta_i(\bx) - \frac{1}{k}\sum_{i \in \calY_k} \heta_i(\bx )}_{\le \frac{1}{k}\sum_{i \in \calY_k} \left |\eta_i(\bx) - \heta_i(\bx) \right | } \\
&  + \underbrace{\frac{1}{k}\sum_{j \in \hat \calY_k} \heta_j(\bx) -  \frac{1}{k}\sum_{j \in \hat \calY_k} \eta_j(\bx)}_{\le \frac{1}{k}\sum_{j \in \hat \calY_k} \left |\heta_j(\bx) - \eta_j(\bx) \right |  }\\
 &  + \underbrace{\frac{1}{k}\sum_{i \in \calY_k} \heta_i(\bx)  - \frac{1}{k}\sum_{j \in \hat \calY_k} \heta_j(\bx)}_{\le 0} 
\end{align*}
Because of the relations given under the braces, we finally get:
$$
\reg_{p@k} (\bh \given \bx) \le 2 \max_{l} \left | \eta_l(\bx) - \heta_l(\bx) \right | \,.
$$
\end{proof}
By getting together both theorems we get an upper bound of the precision@$k$ regret expressed in terms of the regret of the node classifiers. Again, for the zero-regret node classifiers, we get optimal solution in terms of precision@$k$.

\section{Hierarchical softmax with the pick-one-label heuristic}
\label{app:hsm-independent}


\hsmindependent*
\begin{proof}
To proof the proposition it suffices to show that for conditionally independent labels the order of 
labels induced by the marginal probabilities $\eta_j(\bx)$ is the same as the order induced by 
the values of $\eta_j'(\bx)$ obtained by the pick-one-label heuristic (\ref{eq:heuristic}):
\begin{equation*}
\eta_j'(\bx) = \prob'(y_j = 1 \given \bx) = \sum_{\by \in \calY}  \frac{y_j}{\sum_{j'=1}^m y_{j'}}\prob(\by \given \bx).
\end{equation*}
In other words, for any two labels $i, j \in \{1, \dots ,m\}$, $i \neq j$, $\eta_i(\bx) \ge \eta_j(\bx) \Leftrightarrow \eta_i'(\bx) \ge \eta_j'(\bx)$.

Let $\eta_i(\bx) \ge \eta_j(\bx)$. The summation over all $\by$ in (\ref{eq:heuristic}) can be written in the following way:
$$
\eta_j'(\bx) = \sum_{\by \in \calY}  y_j N(\by) \prob(\by | \bx)\,,
$$
where $N(\by) = (\sum_{i=1}^m y_{i})^{-1}$ is a value that depends only on the number of positive labels in $\by$. In this summation we consider four subsets of $\mathcal{Y}$, creating a partition of this set: 
$$
    \mathcal{S}^{u,w}_{i,j} = \{ \by \in \mathcal{Y}: y_i = u \land y_j = w \}, \quad u,w \in \{0, 1\}.
$$
The subset $\mathcal{S}^{0,0}_{i,j}$ does not play any role because $y_i = y_j = 0$ and therefore do not contribute to the final sum.
Then (\ref{eq:heuristic}) can be written in the following way for the $i$-th and $j$-th label:
\begin{eqnarray}
\eta_i'(\bx)  & = & \sum_{\by : \mathcal{S}^{1,0}_{i,j}}{ N(\by) \prob(\by | \bx) } +  \sum_{\by \in \mathcal{S}^{1,1}_{i,j}}{ N(\by) \prob(\by | \bx) } \label{eqn:eta_i}\\
\eta_j'(\bx) & = & \sum_{\by : \mathcal{S}^{0,1}_{i,j}}{ N(\by) \prob(\by | \bx) } +  \sum_{\by \in \mathcal{S}^{1,1}_{i,j}}{ N(\by) \prob(\by | \bx) }
\label{eqn:eta_j}
\end{eqnarray}
The contribution of elements from $\mathcal{S}^{1,1}_{i,j}$ is equal for both $\eta_i'(\bx)$ and $\eta_j'(\bx)$.
It is so because the value of $N(\by) \prob(\by | \bx)$ is the same for all $\by \in \mathcal{S}^{1,1}_{i,j}$: the conditional joint probabilities $\prob(\by | \bx)$ are fixed and they are multiplied by the same factors $N(\by)$.

Consider now the contributions of  $\mathcal{S}^{1,0}_{i,j}$ and  $\mathcal{S}^{0,1}_{i,j}$
to the relevant sums. 
By the definition of $\mathcal{Y}$, $\mathcal{S}^{1,0}_{i,j}$, and $\mathcal{S}^{0,1}_{i,j}$, there exists bijection $b_{i,j}: \mathcal{S}^{1,0}_{i,j} \rightarrow \mathcal{S}^{0,1}_{i,j}$, such that for each $\by' \in \mathcal{S}^{1,0}_{i,j}$ there exists $\by'' \in \mathcal{S}^{0,1}_{i,j}$ equal to $\by'$ except on the $i$-th and the $j$-th position.



Notice that because of the conditional independence assumption the joint probabilities of elements in $\mathcal{S}^{1,0}_{i,j}$ and $\mathcal{S}^{0,1}_{i,j}$ are related to each other. Let $\by'' = b_{i,j}(\by')$, where $\by' \in \mathcal{S}^{1,0}_{i,j}$ and $\by'' \in \mathcal{S}^{0,1}_{i,j}$. The joint probabilities are:
$$
\prob(\by'| \bx) = \eta_i(\bx)(1 - \eta_j(\bx)) \prod_{l \in \mathcal{L} \setminus \{i,j\}} \eta_l(\bx)^{y_l} (1 - \eta_l(\bx))^{1 - y_l}
$$
and
$$
\prob(\by''| \bx) = (1 - \eta_i(\bx)) \eta_j(\bx) \prod_{l \in \mathcal{L} \setminus \{i,j\}} \eta_l(\bx)^{y_l}(1 - \eta_l(\bx))^{1 - y_l}.
$$
One can easily notice the relation between these probabilities: 
$$
\prob(\by'| \bx) = \eta_i(\bx)(1 - \eta_j(\bx)) q_{i,j} \quad \textrm{and} \quad 
\prob(\by''| \bx) = (1 - \eta_i(\bx)) \eta_j(\bx)q_{i,j},
$$
where $q_{i,j} = \prod_{l \in \mathcal{L} \setminus \{i,j\}}\eta_l(\bx)^{y_l} (1 - \eta_l(\bx))^{1 - y_l} \ge 0$.
Consider now the difference of these two probabilities: 
\begin{eqnarray*}
\prob(\by'| \bx) - \prob(\by''| \bx) &=&  \eta_i(\bx)(1 - \eta_j(\bx)) q_{i,j} - (1 - \eta_i(\bx)) \eta_j(\bx)q_{i,j}\\
&=& q_{i,j}( \eta_i(\bx)(1 - \eta_j(\bx)) - (1 - \eta_i(\bx))\eta_j(\bx) ) \\
&=& q_{i,j}(\eta_i(\bx) - \eta_j(\bx)).
\end{eqnarray*}
From the above we see that $\eta_i(\bx) \ge \eta_j(\bx) \Rightarrow  \prob(\by'| \bx) \ge \prob(\by''| \bx)$.
Due to the properties of the bijection $b_{i,j}$, the number of positive labels in $\by'$ and $\by''$ is the same and $N(\by') = N(\by'')$, therefore we also get $\eta_i(\bx) \ge \eta_j(\bx) \Rightarrow \sum_{\by : \mathcal{S}^{1,0}_{i,j}}{ N(\by) \prob(\by | \bx) }  \ge  \sum_{\by : \mathcal{S}^{0,1}_{i,j}}{ N(\by) \prob(\by | \bx) }$, which by (\ref{eqn:eta_i}) and (\ref{eqn:eta_j}) gives us finally $\eta_i(\bx) \ge \eta_j(\bx) \Rightarrow \eta_i'(\bx) \ge \eta_j'(\bx)$.

The implication in the other side, i.e., $\eta_i(\bx) \ge \eta_j(\bx) \Leftarrow  \prob(\by'| \bx) \ge \prob(\by''| \bx)$ holds obviously for $q_{i,j} > 0$. 
For $q_{i,j} = 0$, we can notice, however, that $\prob(\by'| \bx)$ and $\prob(\by''| \bx)$ do not contribute to the appropriate sums as they are zero, and therefore we can follow a similar reasoning as above, concluding that $\eta_i(\bx) \ge \eta_j(\bx) \Leftarrow \eta_i'(\bx) \ge \eta_j'(\bx)$. 

Thus for conditionally independent labels, the order of labels induced by marginal probabilities $\eta_j(\bx)$ is equal to the order induced by $\eta_j'(\bx)$.
As the precision@$k$ is optimized by $k$ labels with the highest marginal probabilities, we have  
that prediction consisted of $k$ labels with highest $\eta_j'(\bx)$ has zero regret for precision@$k$.
\end{proof}

\section{Huffman codes for PLTs}
\label{App:OptComp}
In this section we analyze computational cost of \Algo{PLT}s in binary case, i.e. every inner node has two children. We define the cost of a tree as the total expected fraction of instances which are used in the inner nodes and show that for the multi-class case minimization of this cost coincides with minimization of Huffman criteria. However, in the multi-label case, this does no longer hold. 

We shall use prefix codes, as it is introduced in Section \ref{sec:hsm}, to identify a path from the root to the leaf. Accordingly, a prefix code $\bz = (z_1, \dots , z_\ell ) \in \cal{C}$ determines a path with length $| \bz | = \ell$. The probability to observe a (possibly partial) prefix code is $p_{\bz^i} = \prob ( \bz^i ) $ with respect to the data distribution. If we are given a data that consists of $n$ instances, then the expected number of positive instances that is used to train the node classifier in node $\bz^i$ is $np_{\bz^i}$. Note that in the training process of PLT, only the positive instances shows up in the training data of a child node of $\bz^i$, thus one can define the expected computational cost of a PLT with fixed structure, thus with fixed set of prefix code $\calC$, as
\begin{align}
\sum_{\bz \in \calC} \sum_{i=1}^{| \bz |} p_{\bz^{i-1}} \label{Mishas_cost}.
\end{align}
The next proposition defines the relation of (\ref{Mishas_cost}) and Huffman coding. Huffman coding can be naturally defined in the multi-class case using $\prob (y_i = 1 )$ as weights.
\begin{proposition}
Huffman code minimizes the expected computational cost that is given in (\ref{Mishas_cost}) among binary codes if the data is multi-class.
\end{proposition}

\begin{proof}
Recall that the Huffman code minimizes the following criterion:
\[
\sum_{\bz \in \calC } p_{\bz} | \bz |.
\]
In multi-class setting score $p_{\bz^{i}}$ of each inner node $\bz^{i}$ equals to the sum of scores of its children, and the objective function~given in (\ref{Mishas_cost}) can be rewritten as 
\[
\sum_{\bz \in \calC} \sum_{i=1}^{| \bz |} p_{\bz^{i-1}} = 1 + \sum_{ \bz \in \calC } p_{\bz} \sum_{i=1}^{| \bz |} \mbox{deg} ( \bz^i ) ,
\]
where $\mbox{deg} ( \bz^i ) = \# \{ c : \hat{z}_{i+1} = c, \bz^{i} = \hat{\bz}^{i}, \hat{\bz}\in \calC\}$ is the number of children of the node $\bz^{i}$.
In the case of binary codes $\mbox{deg}(\bz^i)=2$ and
\[
1 + \sum_{ \bz \in \calC } p_{\bz} \sum_{i=1}^{| \bz |} \mbox{deg} ( \bz^i ) = 1 + 2\sum_{ \bz \in \calC } p_{\bz} |\bz |
\]
which completes the proof.
\end{proof}

Note that \cite{Grave_et_al_2017}  considered similar notion of computational cost and optimized in a restricted setup, assuming that the tree depth is at most two. Of course, in practice, the $p_{\bz^i}$ values are not known, but are estimated based on observations from the data distribution. Let us denote an estimate $\hat{p}_{\bz^i}$  of $p_{\bz^i}$. An interesting question to address is that to what extent the empirical computational cost $\sum_{\bz \in \calC} \sum_{i=1}^{| \bz |} \hat{p}_{\bz^{i-1}}$ concentrates around its expected values, and on what parameter of the code/tree structure does depend on.

For multi-label case score of an inner node cannot be represented by a sum of children scores, and this result does not hold.

\section{Pseudocode of PLTs}
\label{app:pseudocode}

The pseudocode below presents the training and prediction procedures of \Algo{PLT}s in detail. \Algo{PLT}s can be trained either in the online or batch mode. Algorithm~\ref{alg:plt-online-learning} follows the former mode used, for example, in \Algo{XT}, the state-of-the-art variant of \Algo{PLT}s described in  Section~\ref{sec:plt-tagging}. Note that learning of feature embeddings used in \Algo{XT} is not included in the pseudocode. 

In Algorithm~\ref{alg:plt-topk-prediction} we present an inference algorithm, which uses the uniform-cost search for finding the top $k$ labels. 
The algorithm searches the tree by starting from the root node $\bz^0$. It uses a priority queue $Q$ to store pairs of the node $\bz^i$ and its probability estimate $\heta_{\bz^i}$. The algorithm pops the element with the highest probability estimate from the queue. If the element is a leaf we add the corresponding label to the final prediction. Otherwise, the algorithm estimates probabilities of the children nodes and adds them into the queue. Once we found the $k$-th label or we run out of the pairs in the queue, we stop the search procedure.


\algnewcommand{\IIf}[1]{\State\algorithmicif\ #1\ \algorithmicthen}
\algnewcommand{\IElse}[1]{\State\algorithmicelse\ #1\ }
\algnewcommand{\IElseIf}[1]{\State\algorithmicelse \algorithmicif\ #1\  \algorithmicthen}
\algnewcommand{\EndIIf}{\unskip\ \algorithmicend\ \algorithmicif}

\newcommand{\Func}[1]{\mathrm{#1}}
\newcommand{\HLFunc}[1]{\mathrm{\textbf{#1}}}
\newcommand{\hby}{\hat{\by}}

\begin{algorithm}[h]
	\caption{Incremental learning of a \Algo{PLT}:}
	\label{alg:plt-online-learning}
	\begin{algorithmic}[1]
	    \Statex \textbf{Input:}
		    \Statex $T$: a label tree with $t$ nodes
		    \Statex $A_{\textrm{online}}$: an incremental learning algorithm
		    \Statex $\mathcal{D}_N$: a set of training examples $(\bx, \by)$
        \Statex

		\Statex \textbf{Output:}
		    \Statex $\mathcal{F}_t = \{f_{\bz^i}\}^t$: a set of $t$ node (binary) classifiers
        \Statex

		\State $\mathcal{F}_t = \emptyset$
		\For{each node $\bz^i \in T$} \Comment{Initialization of binary classifiers}
            \State $\mathcal{F}_t \leftarrow \mathcal{F}_t \cup$ new classifier $f_{\bz^i}$
		\EndFor
		
		\For{each training example $(\bx, \by) \in \mathcal{D}_N$}
		    
		    \If{$\textstyle \sum_{j=1}^m y_j = 0$} \Comment{Select nodes for the positive and negative update}
		        \State $\mathcal{Z}_{\textrm{positive}} \leftarrow \emptyset$
		        \State $\mathcal{Z}_{\lnot \textrm{positive}} \leftarrow \bz^0$
		    \Else
		        \For{each $y_j = 1, j \in \{1, \ldots, m\}$}
		            \State $\bz \leftarrow$ encode $j$
		            \State $\mathcal{Z}_{\textrm{positive}} \leftarrow \mathcal{Z}_{\textrm{positive}} \cup \Big\{ \bigcup \limits_{i = 0} \limits^{l} \bz^i$ \Big\} 
		        \EndFor
		        
		        \For{each $\bz^i \in \mathcal{Z}_{\textrm{positive}}$}
		            \State $\mathcal{Z}_{\lnot \textrm{positive}} \leftarrow \mathcal{Z}_{\lnot \textrm{positive}} \cup \Big\{ \bigcup \limits_{z_i} (\bz^{i-1},z_i) \Big\} \setminus \mathcal{Z}_{\textrm{positive}}$ 
		        \EndFor
		    \EndIf
		
		\For{each node $\bz^i \in \mathcal{Z}_{\textrm{positive}}$} \Comment{Update node classifiers}
		    \State $f_{\bz^i} \leftarrow A_{\textrm{online}}(f_{\bz^i}, \bx, 1)$
		\EndFor
		
		\For{each node $\bz^i \in \mathcal{Z}_{\lnot \textrm{positive}}$}
		    \State $f_{\bz^i} \leftarrow A_{\textrm{online}}(f_{\bz^i}, \bx, 0)$
		\EndFor
		
		\EndFor		
		\State \textbf{return} $\mathcal{F}_t$.
	\end{algorithmic}
\end{algorithm} 

\begin{algorithm}[h!]
	\caption{Top-k prediction with a \Algo{PLT}:}
	\label{alg:plt-topk-prediction}
	\begin{algorithmic}[1]
		\Statex \textbf{Input:}
		    \Statex $T$: a label tree with $t$ nodes
		    \Statex $\mathcal{F}_t = \{f_{\bz^i}\}^t$: a set of $t$ node (binary) classifiers
		    \Statex $\bx$: a test example
		    \Statex $k$: a size of prediction
        \Statex

        \Statex \textbf{Output:}
		    \Statex $\hby$: a vector with top $k$ labels for the test example
        \Statex
		
		\State $\hby  \leftarrow \{0\}^m$
		\State $Q \leftarrow \textsc{PriorityQueue}()$ \Comment{Initialization of priority queue}
		\State $\Func{add}\big(Q,(1, \bz^0)\big)$
		\While{$Q \ne \emptyset$ and $\textstyle \sum_{j=1}^m \hy_j < k$} \Comment{Check if $k$ labels not found}
		    \State $\big(\heta_{\bz^i}, \bz^i \big) \leftarrow \Func{pop}(Q)$
            \If{$i \in \calC$} \Comment{Check if leaf node reached}
                \State $j \leftarrow$ decode $\bz^i$
	            \State $\hy_j \leftarrow 1$
            \Else
                \For{each $z_{i+1}$} \Comment{For each node's children}
                    \State $\bz^{i+1} \leftarrow (\bz^i, z_i)$
                    \State $\Func{add}(Q,(\heta_{\bz^i} \cdot \hprob_{f_{\bz^i}}(z_{i+1} \given \bz^i, \bx), \bz^{i+1}))$
                \EndFor
            \EndIf
        \EndWhile
		\State \textbf{return} $\hby$.
	\end{algorithmic}
\end{algorithm}


\section{Additional experimental results}

\subsection{Comparison of PLTs and HSM on synthetic data}
\label{sec:empirical-synthetic}


In this section we validate our theoretical results presented in Section~\ref{sec:pick-one-label}, which show that  \Algo{HSM} is not amenable to model the marginal probabilities in general for multi-label problems. In this case, \Algo{HSM} with pick-one-label heuristic should be outperformed in terms of precision@$k$ by \Algo{PLT}s which are consistent for this performance measure. To validate this claim empirically, we compare the performance of \Algo{PLT}s and \Algo{HSM} on synthetic datasets of three types: multi-label data with independent labels, multi-label data with conditionally dependent labels and multi-class data.

All synthetic models are based on linear models parametrized by a weight vector $\bw$ of size $d$. The values of the vector are sampled uniformly from a $d$-dimensional sphere of radius 1. Each instance $\bx$, in turn, is represented as a vector sampled from a $d$-dimensional disc of the same radius. 


\paragraph{Multi-class distribution.}

We associate a weigh vector $\bw_j$ with each label $j \in \{1, \ldots, m\}$. The model assigns probabilities to labels at point $\bx$ based on the softmax schema
\begin{equation}
\eta_j(\bx) = \frac{\exp(c\bw_j^\top \bx)}{ \sum_{j' = 1}^m {\exp(c\bw_{j'}^\top \bx)}} 
\label{eq:model_multiclass}
\end{equation}
and draws the positive label according to this probability distribution over labels. Scaling factor $c$ is added to control noise in the model. Higher values of $c$ give less noisy model.


\paragraph{Multi-label distribution with conditionally independent labels.}

The model is similar to the previous one used for the multi-class distribution. The difference lays is normalization as the marginal probabilities do not have to sum up to 1. 
To get a probability of the $j$-th label, we use the logistic transformation:
$$
\eta_j(\bx) = \frac{\exp(\bw_j^\top \bx)}{1 + \exp(\bw_j^\top \bx)}.
$$
Then, we assign a label to an instance based on:
$$
y_j = \assert{ r < \eta_j(\bx) },
$$
where the random value $r$ is sampled uniformly and independently from range $[0,1]$ for each instance $\bx$ and label $j \in \{1, \ldots, m\}$.


\paragraph{Multi-label distribution with conditionally dependent labels.}

To model conditionally dependent labels we use the {mixing matrix} model based
on latent scoring functions generated by  $\mW = (\bw_1, \ldots, \bw_m)$.
The $m \times m$ mixing matrix $\mM$ introduces dependencies between noise $\boldsymbol{\epsilon}$, which stands for the source of randomness in the model. The models $\bw_j$ are sampled from a sphere of radius 1, as in previous cases. The values in the mixing matrix $\mM$ are sampled uniformly and independently from $[-1, 1]$. The random noise vector $\boldsymbol{\epsilon}$ is sampled from $N(0, 0.25)$. The label vector $\by$ is then obtained by element-wise evaluation of the following expression:
$$
\by =  \assert {\mM(\mW^\top\bx + \boldsymbol{\epsilon}) > 0} 
$$
Notice that if $\mM$ was an identity matrix the model would generate independent labels.



\begin{table}[h]
\caption{Means and standard deviations of 50 runs of each experiment. The p-values on the right indicate the significance of the observed differences. The results of \Algo{HSM} and \Algo{PLT}s on multi-class data were always equal.}
\label{tab:synthetic-full}
\centering
\small
\begin{tabular}{l|rr|rr|rrr}
\toprule
 & \multicolumn{2}{c}{HSM} & \multicolumn{2}{|c}{PLT} & \multicolumn{3}{|c}{p-values} \\
 & \multicolumn{1}{l}{mean} & \multicolumn{1}{l}{stdev} & \multicolumn{1}{|l}{mean} & \multicolumn{1}{l}{stdev} & \multicolumn{1}{|l}{t-test} & \multicolumn{1}{l}{sign} & \multicolumn{1}{l}{Wilcoxon} \\
 \midrule
multi-class & 21.90	& 2.74 & 21.90 & 2.74 \\ 
multi-label independent  & 32.57 & 0.34 & 32.58 & 0.33 & 0.4367 & 0.3222 & 0.5980 \\
multi-label dependent   & 70.68 & 5.73 & 71.68 & 5.65 & 9.80E-14 & 3.71E-11 & 3.90E-09 \\
\bottomrule
\end{tabular}
\end{table}

\paragraph{Experimental setting.} 

\Algo{PLT}s and \Algo{HSM} are usually implemented as online learning algorithms, i.e., the node classifiers are updated in an incremental way, example by example. To minimize the impact of the hyperparameter tuning of online algorithms, we have decided to implement batch versions of both algorithms using the \Algo{LibLinear}-based~\citep{liblinear} logistic regression. The sets of training instances for each node classifier are appropriately constructed by taking the corresponding conditioning of the probabilistic models into account. In the case of \Algo{HSM} with the pick-one-label heuristic, we first transform each multi-label example to $s$ multi-class copies of it, one copy for each its label. Each such copy gets then a weight of $1/s$. Such transformation should be concordant with the theoretical model~(\ref{eq:heuristic}).


In the experiments we used the following parameters of the synthetic models: $d = 3$ (i.e., the number of features), $n = 100000$ instances (split $1:1$ for training and test subsets), and $m = 32$ labels or classes. In the case of the multi-class model we report results with the scaling factor $c=10$. The choice of $c$ does not change the interpretation of the results. 
To train logistic regression we use a fixed value of the regularization parameter, standing for a very weak regularization.  For all experiments we report the results in terms of precision$@1$.

\paragraph{Observations.}

Table~\ref{tab:synthetic-full} presents the average results of all experiments besides with the standard deviation of obtained values. As expected the performance of \Algo{PLT}s and \Algo{HSM} on the multi-class data are exactly the same. For the other models, we additionally include the p-values of statistical tests run to verify, whether there is a significant difference in performance between \Algo{PLT}s and \Algo{HSM} with the pick-one-label heuristic. 
In the case of the multi-label data with conditionally independent labels the test shows that there is no evidence to reject the hypothesis that the performance of \Algo{PLT}s and \Algo{HSM} is the same. In the case of multi-label data with conditionally dependent labels, the statistical tests show that \Algo{PLT}s are significantly better than \Algo{HSM} with the pick-one-label heuristic.

\subsection{Comparison of PLTs and HSM on benchmark data}
\label{sec:empirical-benchmark}

In Table~\ref{tab:hsm-vs-plt} we present similar results, but obtained on the benchmark datasets. We use two models. The first one follows the sparse representation and uses an implementation of \Algo{PLT}s and \Algo{HSM} in \Algo{Vowpal Wabbit}~\citep{Langford_et_al_2007}. The second one is based on \Algo{fastText} and produces the dense representation. In all models we use Huffman trees. The results clearly indicate the better performance of \Algo{PLT}s over \Algo{HSM}.

\begin{table*}[ht!]
        \caption{Precision@$k$ with $k=\{ 1,3,5\}$ of a simple \Algo{HSM} and a simple \Algo{PLT} implementaions. }
        \label{tab:hsm-vs-plt}
        \begin{center}
                \begin{footnotesize}
                        \vspace{-0.1cm}
						\resizebox{\textwidth}{!}{
                        \begin{tabular}{l | r@{\hskip 0pt} | r@{\hskip 0pt} | r@{\hskip 0pt} | r@{\hskip 0pt} | l | r@{\hskip 0pt} | r@{\hskip 0pt} | r@{\hskip 0pt} | r@{\hskip 0pt} }
                                \toprule
	& \multicolumn{2}{|c|}{\Algo{Vowpal Wabbit}} & \multicolumn{2}{|c|}{\Algo{fastText}} & & \multicolumn{2}{|c|}{\Algo{Vowpal Wabbit}} & \multicolumn{2}{|c}{\Algo{fastText}} \\
	\multicolumn{1}{c|}{Dataset} & \multicolumn{1}{|c|}{\Algo{HSM}} & \multicolumn{1}{|c}{\Algo{PLT}} & \multicolumn{1}{|c|}{\Algo{HSM}} & \multicolumn{1}{|c|}{\Algo{PLT}} & 
	\multicolumn{1}{c|}{Dataset} 
        & \multicolumn{1}{|c|}{\Algo{HSM}} & \multicolumn{1}{|c}{\Algo{PLT}} & \multicolumn{1}{|c|}{\Algo{HSM}} & \multicolumn{1}{|c}{\Algo{PLT}} \\ \midrule
	\textbf{EUR-Lex} & \data{56.98}{46.99}{39.09} & \databf{74.55}{60.60}{50.05} & \data{66.39}{54.05}{44.73} & \databf{73.19}{57.79}{46.98} & 
	\textbf{AmazonCat-13K} & \data{86.69}{72.00}{57.97} & \databf{91.46}{76.00}{61.40} & \data{90.18}{72.53}{56.20} & \databf{92.98}{75.75}{59.53} \\ \midrule
	\textbf{Wiki-30K} & \data{70.20}{60.11}{53.17} & \databf{84.34}{72.34}{62.72} & \data{83.02}{69.66}{59.50} & \databf{85.11}{73.12}{62.67} & 
	\textbf{Delicious-200K} & \data{41.58}{33.24}{28.04} & \databf{45.27}{38.95}{35.59} & \data{42.17}{37.94}{35.77} & \databf{46.98}{40.99}{38.04} \\ \midrule
	\textbf{WikiLSHTC-325K} & \data{36.90}{22.30}{16.60} & \databf{41.63}{26.78}{20.39} & \data{41.28}{24.68}{18.08} & \databf{41.78}{24.96}{18.53} & 
    \textbf{Amazon-670K} & \data{33.64}{28.58}{25.01} & \databf{36.85}{32.48}{29.15} & \data{25.04}{21.06}{18.28} & \databf{26.18}{22.76}{20.29} \\
								\bottomrule
                        \end{tabular}}
                \end{footnotesize}
        \end{center}
\end{table*}

\subsection{The ablation analysis on benchmark datasets}
\label{app:ablation-analysis}

Table~\ref{tab:ablation-analysis} contains results of the ablation analysis in which we compare different components of the \Algo{XT} algorithm. We analyze the influence of the Huffman tree vs. top-down clustering, the simple averaging of features vectors vs. the TF-IDF-based weighting, and no regularization vs. L2 regularization. For every configuration, we conducted a grid search of hyperparameters from ranges reported in Appendix~\ref{sec:hyper}. The results clearly show that the components need to combined together to obtain the best results. The best combination is usually the one that uses top-down clustering, TF-IDF-based weighting, and L2 regularization. It is worth to notice that top-down clustering alone gets worse results than Huffman trees with TF-IDF-based weighting and L2 regularization. 




\begin{table*}[h]
\vspace{-.4cm}
        \caption{\small Precision@$k$ scores with $k=\{ 1,3,5\}$ of different variants of \Algo{PLT} in \Algo{fastText} (\Algo{extremeText})}
        \label{tab:ablation-analysis}
        \begin{center}
                \vspace{-.2cm}
                \begin{footnotesize}
                        \resizebox{\textwidth}{!}{
                        \begin{tabular}{l | l@{\hskip 0pt} | l@{\hskip 0pt} | l@{\hskip 0pt} | l@{\hskip 0pt} | l@{\hskip 0pt} || l@{\hskip 0pt} | l@{\hskip 0pt} | l@{\hskip 0pt} | l@{\hskip 0pt} }
                                \toprule
\multicolumn{1}{c|}{Dataset} 
& \multicolumn{1}{|c|}{Metrics} 
& \multicolumn{1}{|c|}{Huff.}
& \multicolumn{1}{|c|}{\datatwo{Huff. +}{TF-IDF}}
& \multicolumn{1}{|c|}{Huff. + L2}
& \multicolumn{1}{|c||}{\datatwo{Huff. + L2}{+ TF-IDF}}
& \multicolumn{1}{|c|}{Clus.}
& \multicolumn{1}{|c|}{\datatwo{Clus. +}{TF-IDF}}
& \multicolumn{1}{|c|}{Clus. + L2}
& \multicolumn{1}{|c}{\datatwo{Clus. + L2}{+ TF-IDF}} \\

\midrule

\textbf{Eurlex}
& \data{P@1}{P@3}{P@5}
& \data{63.39}{50.48}{41.19}
& \data{71.20}{57.26}{46.93}
& \data{62.79}{48.99}{40.16}
& \data{74.60}{60.36}{49.72}
& \data{68.31}{54.62}{45.23}
& \data{75.05}{61.65}{50.99}
& \data{65.05}{50.87}{41.71}
& \data{\textbf{77.68}}{\textbf{63.37}}{\textbf{52.85}} \\
\midrule

\textbf{AmazonCat-13K}
& \data{P@1}{P@3}{P@5}
& \data{90.10}{72.67}{57.69}
& \data{89.19}{72.81}{58.30}
& \data{90.84}{73.10}{57.69}
& \data{91.08}{75.27}{60.34}
& \data{72.95}{61.66}{48.38}
& \data{77.13}{65.76}{52.72}
& \data{91.73}{75.07}{59.63}
& \data{\textbf{92.43}}{\textbf{77.65}}{\textbf{62.74}} \\
\midrule

\textbf{Wiki-30K}
& \data{P@1}{P@3}{P@5}
& \data{78.04}{63.31}{52.65}
& \data{78.14}{67.47}{56.00}
& \data{82.28}{68.75}{58.33}
& \data{\textbf{85.23}}{\textbf{73.52}}{\textbf{63.71}}
& \data{78.40}{65.62}{55.61}
& \data{78.69}{66.40}{56.77}
& \data{82.13}{69.22}{58.69}
& \data{85.21}{73.18}{63.39} \\
\midrule

\textbf{Delicious-200K}
& \data{P@1}{P@3}{P@5}
& \data{45.71}{40.69}{38.02}
& \data{47.24}{40.88}{37.74}
& \data{46.48}{41.38}{38.78}
& \data{\textbf{47.85}}{\textbf{42.08}}{\textbf{39.13}}
& \data{44.58}{40.42}{38.15}
& \data{45.13}{40.63}{38.23}
& \data{46.55}{41.08}{38.25}
& \data{47.31}{41.63}{38.88} \\
\midrule

\textbf{WikiLSHTC-325K}
& \data{P@1}{P@3}{P@5}
& \data{36.23}{20.60}{14.70}
& \data{38.28}{22.43}{16.33}
& \data{41.10}{25.62}{19.10}
& \data{42.83}{26.33}{19.45}
& \data{34.39}{21.17}{15.09}
& \data{39.66}{24.79}{18.28}
& \data{54.95}{36.42}{27.25}
& \data{\textbf{58.73}}{\textbf{39.24}}{\textbf{29.26}} \\
\midrule

\textbf{Wiki-500K}
& \data{P@1}{P@3}{P@5}
& \data{41.01}{24.79}{18.64}
& \data{40.62}{25.68}{19.74}
& \data{41.55}{25.94}{19.79}
& \data{49.80}{32.39}{24.62}
& \data{46.63}{31.59}{24.30}
& \data{49.20}{34.52}{26.83}
& \data{56.04}{38.87}{30.46}
& \data{\textbf{64.48}}{\textbf{45.84}}{\textbf{35.46}} \\
\midrule

\textbf{Amazon-670K}
& \data{P@1}{P@3}{P@5}
& \data{23.19}{19.80}{17.48}
& \data{29.70}{25.84}{22.96}
& \data{21.64}{18.11}{15.83}
& \data{32.11}{27.77}{24.64}
& \data{28.54}{25.52}{23.30}
& \data{36.24}{32.15}{29.19}
& \data{28.95}{25.74}{23.22}
& \data{\textbf{39.90}}{\textbf{35.36}}{\textbf{32.04}} \\
                                \bottomrule
                        \end{tabular}}
                \end{footnotesize}
        \end{center}
\end{table*}

\subsection{Tuning of hyperparameters}
\label{sec:hyper}

The \Algo{PLT} has only one global hyperparameter which is the degree of the tree denoted by $b$.
The other hyperparameters are associated with the node classifiers.
The \Algo{HSM} and \Algo{PLT} in Vowpal Wabbit was tuned with the stochastic gradient descent with a step size $\eta_t$ calculated separately for each node
according to $\eta_t = \eta \times (1 / t)^p$ where $t$ is number of node updates and $\eta$ and $p$ are hyperparameters.
In \Algo{fastText}-based methods, \Algo{HSM}, \Algo{Learned Tree} and \Algo{XT}, $\eta_t$ decreased linearly during training from $\eta$ to 0.0. 
In the \Algo{XT} the optimization methods has been extended by L2 regularization, so it has one additional parameter. Balanced k-means clustering used to build a tree in \Algo{XT} has also a stopping parameter $\epsilon$ set by default to $0.001$.  

\vspace{\tableBefore}
\begin{table}[ht!]
\caption{The hyperparameters of the \Algo{HSM} and \Algo{PLT} methods and their ranges used in hyperparameter optimization. Notation: $b$ -- tree arity, $\eta$ -- learning rate}
\label{tab:hyppar-hsm-plt}
\begin{center}
\begin{tabular}{c|c}
\hline
Hyperparameter & Range \\
\hline
$b$ & $\{ 2, \dots , 32\}$ \\
$\eta$ & $[0.0001 - 1.0]$ \\
number of epochs & $\{ 20, 30, 40\}$ \\
\hline
\end{tabular}
\end{center}
\end{table}
\vspace{\tableAfter}

\vspace{\tableBefore}
\begin{table}[ht!]
\caption{The hyperparameters of the \Algo{fastText} and \Algo{Learned Tree} methods and their ranges used in hyperparameter optimization. Notation: $b$ -- tree arity, $\eta$ -- learning rate}
\label{tab:-ft}
\begin{center}
\begin{tabular}{c|c}
\hline
Hyperparameter & Range \\
\hline
$b$ & $\{ 2, \dots , 32\}$ \\
$\eta$ & $[0.0001 - 1.0]$ \\
number of epochs & $\{ 20, 30, 40\}$ \\
dim & $\{ 500\}$ \\
\hline
\end{tabular}
\end{center}
\end{table}
\vspace{\tableAfter}

\vspace{\tableBefore}
\begin{table}[ht!]
\caption{The hyperparameters of the \Algo{XT} method and their ranges used in hyperparameter optimization. Notation: $b$ -- tree arity (number of centroids used in k-means clustering), $\epsilon$ -- stoping condition of k-means clustering, $\eta$ -- learning rate}
\label{tab:hyppar-xt}
\begin{center}
\begin{tabular}{c|c}
\hline
Hyperparameter & Range \\
\hline
$b$ & $\{ 2 \}$ \\
$\epsilon$ & $\{0.001\}$ \\
$\eta$ & $\{0.5, 0.1, 0.05 \}$ \\
L2 regularization & $\{0.001, 0.002, 0.003\}$ \\
number of epochs & $\{ 20, 30, 40 \}$ \\
dim & $\{ 500 \}$ \\
max leaves & $\{ 100 \}$ \\
\hline
\end{tabular}
\end{center}
\end{table}
\vspace{\tableAfter}

\section{Information about the benchmark datasets}
\label{app:benchmark_datasets}

Table~\ref{tbl:benchmark-datasets} contains the basic statistics of the benchmark datasets used in the experiments taken from the Extreme Classification Repository: \url{http://manikvarma.org/downloads/XC/XMLRepository.html}.

\begin{table}[h]
\caption{Statistics of the benchmark datasets.}
\label{tbl:benchmark-datasets}
\begin{center}
\begin{small}
\begin{tabular}{l rrrrrr}
\toprule
Dataset        & \multicolumn{1}{c}{\begin{tabular}[c]{@{}c@{}}Number of \\ features\end{tabular}} & \multicolumn{1}{c}{\begin{tabular}[c]{@{}c@{}}Number of \\ labels\end{tabular}} & \multicolumn{1}{c}{\begin{tabular}[c]{@{}c@{}}Number of \\ train points\end{tabular}} & \multicolumn{1}{c}{\begin{tabular}[c]{@{}c@{}}Number of \\ test points\end{tabular}} & \multicolumn{1}{c}{\begin{tabular}[c]{@{}c@{}}Avg. points \\ per label\end{tabular}} & \multicolumn{1}{c}{\begin{tabular}[c]{@{}c@{}}Avg. labels \\ per point\end{tabular}} \\
\midrule
EURLex-4K &	5000 & 3993 &	15539 &	3809 &	25.73& 	5.31 \\
AmazonCat-13K &	203882	& 13330	& 1186239	& 306782	& 448.57 &	5.04 \\
Wiki10-31K     & 101938  & 30938 & 14146 & 6616   & 8.52 & 18.64  \\
Delicious-200K & 782585 & 205443 & 196606 & 100095 & 72.29 & 75.54 \\
WikiLSHTC-325K & 1617899 & 325056      & 1778351 & 587084 & 17.46 & 3.19 \\
Wikipedia-500K & 2381304 & 501070 & 1813391 & 783743 & 24.75 & 4.77 \\
Amazon-670K    & 135909 & 670091     & 490449 & 153025  & 3.99 & 5.45   \\
\bottomrule 
\end{tabular}
\end{small}
\end{center}
\end{table}

\end{document}